\pdfoutput=1
\documentclass[twoside]{article}
\usepackage[accepted]{aistats2020}

\usepackage{amsfonts}
\usepackage{amsthm}
\usepackage{amssymb}
\usepackage{xcolor}
\usepackage{graphicx}
\usepackage{natbib}

\newtheorem{theorem}{Theorem}[section]
\newtheorem{corollary}{Corollary}[theorem]
\newtheorem{lemma}[theorem]{Lemma}
\newtheorem{definition}{Definition}[section]
\newtheorem*{remark}{Remark}
\newtheorem{assumption}{Assumption}
\newcommand{\supp}{\textbf{supp}\ }
\newcommand{\Span}{\textbf{span}}

\begin{document}
\runningtitle{The Expressive Power of a Class of Normalizing Flow Models}

\twocolumn[
\aistatstitle{The Expressive Power of a Class of Normalizing Flow Models}
\aistatsauthor{Zhifeng Kong \And Kamalika Chaudhuri}
\aistatsaddress{z4kong@eng.ucsd.edu \\ University of California San Diego \And kamalika@cs.ucsd.edu\\ University of California San Diego}]

\begin{abstract}
  Normalizing flows have received a great deal of recent attention as they allow flexible generative modeling as well as easy likelihood computation. While a wide variety of flow models have been proposed, there is little formal understanding of the representation power of these models.
  In this work, we study some basic normalizing flows and rigorously establish bounds on their expressive power. Our results indicate that while these flows are highly expressive in one dimension, in higher dimensions their representation power may be limited, especially when the flows have moderate depth.
\end{abstract}

\section{Introduction}

Normalizing flows are a class of deep generative models that aspire to learn an invertible transformation to convert a pre-specified distribution, such as a Gaussian, to the distribution of the input data. These models offer flexible generative modeling -- as the invertible transformation can be implemented by deep neural networks -- and easy likelihood computation in equation \eqref{eqn: likelihood computation chain rule} that follows from the invertibility of the transformation \citep{rezende2015variational}.

Due to these advantages and their empirical success, a number of flow models have been proposed \citep{dinh2014nice, germain2015made, uria2016neural, kingma2016improved, tomczak2016improving, dinh2016density, papamakarios2017masked, huang2018neural, berg2018sylvester, grathwohl2018ffjord,behrmann2018invertible, jaini2019sum, ho2019flow++}. However, the expressive power offered by different kinds of flow models -- what kind of distributions they can map between, and with what complexity -- remains not well-understood, which makes it challenging to select the right flow model for specific tasks. Obviously, due to their invertible nature, a normalizing flow can only transform a distribution to one with a homeomorphic support \citep{armstrong2013basic}. However, even within such distributions, it remains unclear whether a simple distribution supported on $\mathbb{R}^d$ could be transformed or approximated via a normalizing flow from a Gaussian.

In this work, we carry out a rigorous analysis of the expressive power of planar flows, Sylvester flows, and Householder flows -- the most basic classes of normalizing flows. The main challenge in analyzing the expressive power of any flow model class is \textit{invertibility}. There is a body of prior work that analyzes the universal approximation properties of standard neural networks; however, analyzing the approximation properties of \textit{invertible} mappings between distributions is a completely different problem. Just because a function class $\mathcal{F}$ is a universal approximator does not mean that the set of all its invertible functions can transform between arbitrary distributions; dually, even if functions in $\mathcal{F}$ have limited expressivity, it is possible that its invertible subset is an universal approximator in transforming between distributions \citep{villani2008optimal}. Additionally, universal approximation properties are often proved by construction via non-invertible functions \citep{lu2017expressive,lin2018resnet} and hence these constructions cannot to be used to establish properties of the corresponding flows.

This work gets around this challenge by studying properties of input-output distribution pairs directly, instead of considering the transformation class itself. In particular, we consider both a local and global analysis of properties of planar flows, their higher dimensional generalization -- Sylvester flows, and Householder flows. First, we analyze the local topology -- namely, the directional derivatives of the induced density. Second, we seek to bound the global total variation distance between the input and output distributions that can be achieved by each planar flow or Householder flow under certain conditions.

Using these two kinds of analysis, we make three main contributions in this paper.

First, we show that in one dimension, even planar flows are highly expressive. In particular, they can transform a source distribution supported on $\mathbb{R}$ to an arbitrarily-accurate approximation of any target distribution supported on a finite union of intervals. The conclusion holds even if we restrict to planar flows with ReLU non-linearity and Gaussian source distributions. This indicates that planar flows in one dimension are universal approximators.

We next turn our attention to general $d$-dimensional spaces, and we look at what kinds of distributions may be expressed by a Sylvester flow model acting on a Gaussian, mixtures of Gaussian (MoG) distributions, or product (Prod) distributions. We show that when the non-linearity is a ReLU function, Sylvester flows of any depth cannot in general exactly transform between certain standard classes of distributions. In particular, ReLU Sylvester flows cannot exactly transform any mixture of $k$ Gaussian distributions or product distributions into another one -- no matter what the depth is -- except under very special circumstances.

Finally, we consider the approximation capability of normalizing flow models in $d$-dimensional space. Here, we focus on local planar flows with a class of local non-linearities -- including common non-linearities such as $\tanh$, $\arctan$ and sigmoid -- and Householder flows. We show that in these cases, provided certain conditions hold, transforming a source distribution into a target may require flows of inordinately large depth. In particular, if the target distribution $p(z)$ is constant in a ball centered at the origin and proportional to $\exp(-\|x\|_2^{1/\tau})$ outside the ball, then $p$ may require local planar flows with depth $\Omega\left(d^{1/\tau-1}\right)$ to transform from an arbitrary source distribution (that is not too close). A similar conclusion holds for Householder flows when the target distribution is close to the standard Gaussian distribution. These results indicate that when local planar flows with certain non-linearities and Householder flows have moderate depth, they may have poor approximation power.

\subsection{Related Work}

There is a body of work on analyzing the approximation properties of neural networks \citep{cybenko1989approximation, hornik1989multilayer, hornik1991approximation, montufar2014number, telgarsky2015representation, lu2017expressive, hanin2017universal, raghu2017expressive}. Most of these results apply to feed-forward neural networks including non-invertible functions. Therefore, their universal approximation properties do not directly translate to normalizing flows.

The work most related to ours shows that a residual network (ResNet) in which each block is a single-neuron hidden layer with ReLU activation is a universal approximator in the space of Lebesgue integrable functions from $\mathbb{R}^d$ to $\mathbb{R}^d$ \citep{lin2018resnet}. This is related to us because the set of all such ResNets with $T$ invertible blocks is exactly $T$-layer ReLU planar flows. However, their construction that establishes this property is based on non-invertible mappings, consequently, their universal approximation result does not extend to planar flows.


There has also been some recent related work on the expressive power of generative networks. In particular, it was proved by construction that when the output dimension is equal to the input dimension, deep neural networks can approximately transform Gaussians to uniform distributions and vice versa \citep{bailey2018size}. However, their constructions are again based on non-invertible functions, and hence their results do not extend to normalizing flows.

Finally, there is also a body of empirical work on different kinds of normalizing flows; a more detailed discussion of these works is presented in Section~\ref{sec:morerelatedwork}.



\section{Preliminaries}
\subsection{Definitions and Notation}
\label{sec: def of flows}
Suppose $d$ is the data dimension. Let $z\in\mathbb{R}^d$ be a random variable with density $q_z:\mathbb{R}^d\rightarrow\{0\}\cup\mathbb{R}^+$. Then, an invertible function $f:\mathbb{R}^d\rightarrow\mathbb{R}^d$ is called a \textit{normalizing flow} if $f$ is differentiable almost everywhere $(a.e.)$ and the determinant of the Jacobian matrix of $f$ does not equal to zero:
\[\det J_f(z)\neq0\ (a.e.) \]
where $J_f(z)_{ij}=\frac{\partial f_i}{\partial z_j},\ \forall i,j\in\{1,\cdots,d\}$. If we apply a flow $f$ over $z$, we obtain a new random variable $y=f(z)$, whose density $q_y$ can be written through the change-of-variable formula:
\begin{equation}
\label{eqn:change of variable}
q_y(y)=\frac{q_z(z)}{|\det J_f(z)|}
\end{equation}
or
\begin{equation}
\label{eqn:change of variable log}
\log q_y(y)=\log q_z(z)-\log|\det J_f(z)|
\end{equation}
For conciseness, we write $q_y=f\#q_z$ in such context.
In particular, if the flow $f$ is composed of $T$ \textit{simple} flows $f_t,t=1\cdots,T$:
\[f=f_T\circ f_{T-1}\circ \cdots\circ f_1\]
then according to the chain rule of the Jacobian matrix, we have
\begin{equation}
\label{eqn: likelihood computation chain rule}
    \log q_y(y)=\log q_z(z)-\sum_{t=1}^T \log|\det J_{f_t}(z_{t-1})|
\end{equation}
where $z_0=z,\ z_t=f_t(z_{t-1}),\ t=1,\cdots,T$.

Two simple flows are defined below \citep{rezende2015variational}:

\textbf{Planar Flows}. Given the scaling vector $u\in\mathbb{R}^d$, tangent vector $w\in\mathbb{R}^d$, shift $b\in \mathbb{R}$, and non-linearity $h:\mathbb{R}\rightarrow\mathbb{R}$, a planar flow $f_{\mathrm{pf}}$ on $\mathbb{R}^d$ is defined by
\begin{equation}
\label{eqn: planar flow}
f_{\mathrm{pf}}(z)=z+uh(w^{\top}z+b)
\end{equation}

\textbf{Radial Flows}. Given the smoothing factor $a\in\mathbb{R}^+$, scaling factor $b\in\mathbb{R}$, and center $z_0\in\mathbb{R}^d$, a radial flow $f_{\mathrm{rf}}$ on $\mathbb{R}^d$ is defined by
\begin{equation}
\label{eqn: radial flow}
f_{\mathrm{rf}}(z)=z+\frac{b}{a+\|z-z_0\|_2}(z-z_0)
\end{equation}

A geometric intuition between planar and radial flows is shown in Section \ref{sec: planad vs radial}. Planar flows can be generalized to a higher dimension below \citep{berg2018sylvester}:

\textbf{Sylvester Flows}. Given the flow dimension $m<d$, scaling matrix $A\in\mathbb{R}^{d\times m}$, tangent matrix $B\in\mathbb{R}^{d\times m}$, shift vector $b\in \mathbb{R}^d$, and non-linearity $h:\mathbb{R}\rightarrow\mathbb{R}$, a Sylvester flow $f_{\mathrm{syl}}$ on $\mathbb{R}^d$ is defined by
\begin{equation}
\label{eqn: sylvester flow}
f_{\mathrm{syl}}(z)=z+Ah(B^{\top}z+b)
\end{equation}
where $h$ maps coordinate-wise.

In addition, Householder matrices can also be used to construct flows \citep{tomczak2016improving}:

\textbf{Householder Flows}. Given a unit reflection vector $v\in\mathbb{R}^d$, a Householder flow $f_{\text{hh}}$ on $\mathbb{R}^d$ is defined by
\begin{equation}
\label{eqn: householder flow}
f_{\text{hh}}(z)=z-2vv^{\top}z
\end{equation}

For conciseness, we denote these flows by \textit{base} flows.

\subsection{Problem Statement}
In this paper, we study the expressivity of base flows in Section \ref{sec: def of flows}: given an input distribution $q$, we hope to understand when a flow $f$ composed of a finite number of base flows can transform $q$ into any target distribution $p$ or its approximation on $\mathbb{R}^d$. Formally, suppose $f$ is composed of $T$ base flows in the same class. We propose to answer the following two questions:

\textbf{Q}1 (Exact transformation):
Under what conditions is it possible to \textit{exactly} transform $q$ into $p$ with a finite number of base flows? That is, $f\#q=p,\ (a.e.)$.

\textbf{Q}2 (Approximation):
Since sometimes it may not be possible to exactly transform $q$ into $p$, when is it possible to \textit{approximate} $p$ in total variation distance (which is equal to half of the $\ell_1$ distance)? How many layers of base flows do we need? That is, given $\epsilon>0$, is there a bound for $T$ such that
\[\|f\#q-p\|_1\leq\epsilon\]

\subsection{Additional Definitions and Notations}
The determinant of the Jacobian matrix of a planar flow $f_{\mathrm{pf}}$, a Sylvester flow $f_{\mathrm{syl}}$, and a Householder flow $f_{\text{hh}}$ can be easily calculated by
\begin{equation}\label{eqn:det J_f h}
\begin{array}{rl}
     \det J_{f_{\mathrm{pf}}}(z)&=1+u^{\top}w h'(w^{\top}z+b)  \\
     \det J_{f_{\mathrm{syl}}}(z)&=\det (I_m+diag(h'(B^{\top}z+b))B^{\top}A)\\
     \det J_{f_{\text{hh}}}(z)&=-1
\end{array}
\end{equation}

In this paper, we consider three types of non-linearities $h$: $\text{ReLU}(x)=\max(x,0)$, general differentiable functions, and local non-linearities (see Section \ref{sec:high d approx} for detail) including $\tanh(x)$, $\arctan(x)$ and $\text{sigmoid}(x)=1/(1+\exp(-x))$. Specifically, let $h=$ ReLU and $1\{\cdot\}$ be the indicator function, then $\det J_{f_{\mathrm{pf}}}$ is equal to
\begin{equation}\label{eqn:det J_f ReLU}
\det J_{f_{\mathrm{pf}}}(z)=
1+u^{\top}w\cdot1\{w^{\top}z+b\geq 0\}
\end{equation}
A ReLU planar/Sylvester flow is invertible under certain bounds on its parameters as ReLU is Lipschitz.

We make a few additional definitions here.
$\mathcal{N}$ denotes a Gaussian distribution on $\mathbb{R}^d$:
\[\mathcal{N}(x;\mu,\Sigma)=\frac{\exp\left(-\frac12 (x-\mu)^{\top}\Sigma^{-1}(x-\mu)\right)}{(2\pi)^{d/2}\sqrt{\det\Sigma}}\]
The set $\supp p$ denotes the support of distribution $p$:
\[\supp p=\{x\in\mathbb{R}^d:p(x)>0\}\]
For vectors $w_i\in\mathbb{R}^d, 1\leq i\leq k$, the $\Span$ of them denotes the subspace spanned by $\{w_i\}_{i=1}^k$:
\[\Span\{w_1,\cdots,w_k\}=\left\{\sum_{i=1}^k \alpha_i w_i:\ \alpha_i\in\mathbb{R},1\leq i\leq k\right\}\]
The $\Span$ of a set of matrices is defined as the span of the union of their column vectors.
For any differentiable function $g:\mathbb{R}^d\rightarrow\mathbb{R}$ and direction $\delta\in\mathbb{R}^d\setminus\{0\}$, its corresponding directional derivative is defined by
\[
\lim_{\alpha\rightarrow0}\frac{g(x+\alpha\delta)-g(x)}{\alpha}=\nabla_x g(x)^{\top}\delta\]

\subsection{Challenges}

The main challenge in analyzing whether a class of flows can universally approximate any target distribution when applied to a fixed source is \textit{invertibility}. To understand this, suppose $\mathcal{F},\mathcal{C}$ are function classes and $\mathcal{I}$ is the set of all invertible functions.

Even if $\mathcal{F}$ can approximate any function in $\mathcal{C}$, it might not hold that the invertible functions in $\mathcal{F}$ can approximate any invertible function in $\mathcal{C}$.
This is because the set of invertible functions $\mathcal{I}$ might have no interior in $\mathcal{C}$: for any invertible function, it is possible to modify it slightly to make it non-invertible -- and hence the approximation to an invertible function $c \in \mathcal{C}$ may be a non-invertible function $f \in \mathcal{F}$ (see \textbf{Lemma} 4, \citep{mulansky1998interpolation}).
For instance, it was shown that a certain ResNet ($\mathcal{F}$) is a universal approximator in $\mathcal{C}=\ell_1(\mathbb{R}^d)$ \citep{lin2018resnet}, and its invertible function subset ($\mathcal{F}\cap\mathcal{I}$) is exactly the set of transformations composed of finitely many ReLU planar flows. However, since the universal approximation property was proved by construction using the non-invertible trapezoid functions, this result does not translate to ReLU planar flows.

Dually, if $\mathcal{F}$ has limited expressivity, it might still happen that functions in $\mathcal{F}\cap\mathcal{I}$ can approximate or even express transformations between arbitrary pairs of distributions. This is because a small subset of functions $\mathcal{T}$ (for instance, increasing triangular maps \citep{villani2008optimal}) is enough to transform between distributions. Therefore, if $\mathcal{F}\cap\mathcal{I}$ is dense in $\mathcal{T}$, then it is expressive. It is however challenging to find all such dense sets $\mathcal{T}$.


\section{The $d=1$ case}
\label{sec:1d}
In this section, we discuss the universal approximation properties of Sylvester flows when the data dimension $d=1$. In this case, a Sylvester flow is identical to a planar flow. However, the one-dimensional case is not trivial and requires delicate design. For both general and ReLU non-linearity cases, we demonstrate they are able to achieve universal approximation.

\subsection{General Smooth Non-linearity}
\label{sec:1d h}
Suppose the flow $f$ is a single planar flow with an arbitrary smooth non-linearity $h$. It is straightforward to show by construction that if $\supp p=\supp q=\mathbb{R}$, then there exists a planar flow that exactly transforms $q$ into $p$. (See \textbf{Lemma} \ref{lemma:1d general}). Using these exact transformations, we can approximate any density supported on a finite union of intervals when the input distribution is supported on $\mathbb{R}$ (e.g. a Gaussian).

\begin{theorem}[Universal Approximation]
\label{thm:1d approx}
Let $p,q$ be densities on $\mathbb{R}$ such that $p$ is supported on a finite union of intervals and $\supp q = \mathbb{R}$. Then, for any $\epsilon>0$, there exists a planar flow $f_{\mathrm{pf}}$ such that $\|f_{\mathrm{pf}}\#q-p\|_1\leq\epsilon$.
\end{theorem}

Since in \textbf{Theorem} \ref{thm:1d approx}, the support of $p$ might not be $\mathbb{R}$, we are unable to achieve exact transformation between $p$ and $q$. However, approximation is possible in that we can transform $q$ into $\tilde{p}$, a distribution supported on $\mathbb{R}$ but approximates $p$ in $\ell_1$ norm. To achieve this, we construct such $\tilde{p}$ that satisfying $\tilde{p}\approx p$ on $\supp p$ and $\tilde{p}\approx0$ on $\overline{\supp p}$. An example is shown in Figure \ref{fig: 1d general approx}, where $p(x)=\frac34\min((|x|-1)^2, (|x|-3)^2)$ for $1\leq|x|\leq3$ and $p(x)=0$ elsewhere.

\begin{figure}[!h]
    \centering
    \includegraphics[trim=10 0 10 10, clip, width=0.45\textwidth]{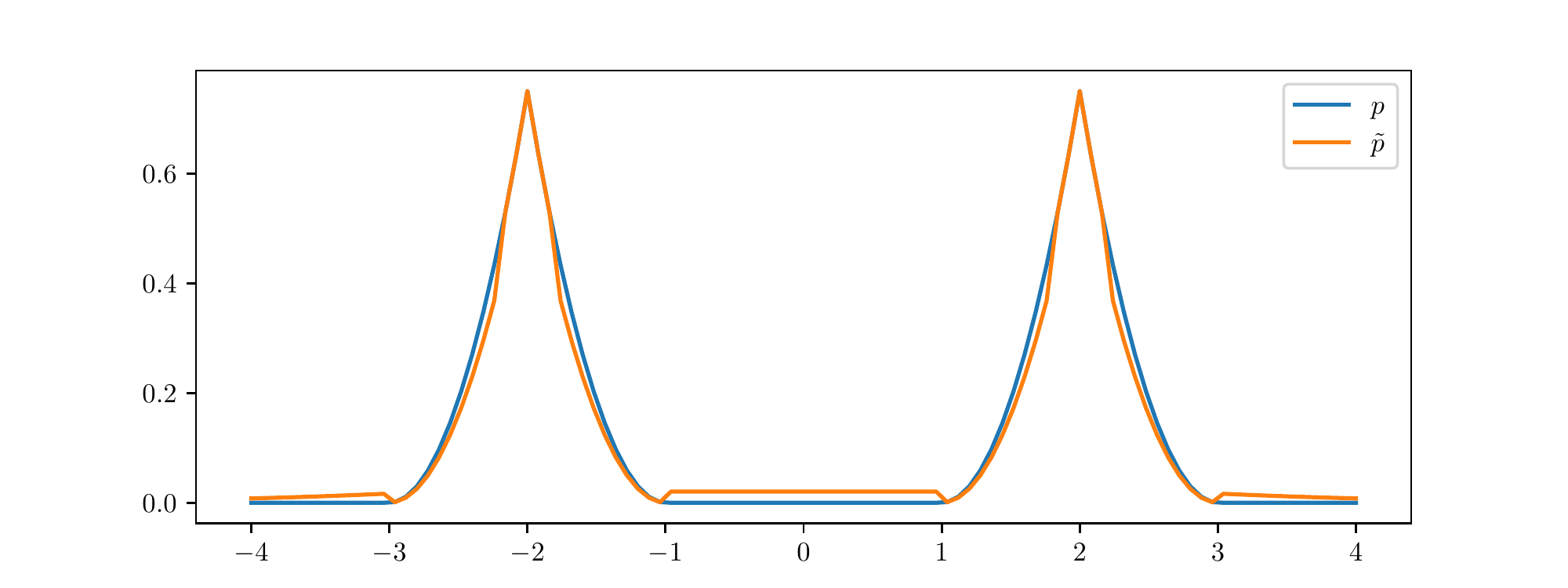}
    \caption{Target distribution $p$ and its approximation $\tilde{p}$ with $\supp\tilde{p}=\mathbb{R}$.}
    \label{fig: 1d general approx}
\end{figure}

\subsection{ReLU Non-linearity}
\label{sec:1d ReLU}
Since the ReLU activation has been proven to be expressive and is popular in recent neural network models \citep{he2016deep, lin2018resnet}, we provide a universal approximation result for planar flows with ReLU non-linearity.

Suppose the one-dimensional ReLU flow has the form $f(z)=f_{\mathrm{pf}}(z)=z+uh(wz+b)$, where $h=\text{ReLU}$. Since ReLU is linear on both $\mathbb{R}^-$ and $\mathbb{R}^+$, we assign $u=\pm1$ for concreteness. In addition, to ensure the transformation is strictly increasing, we require $uw>-1$. Different from the general non-linearity case, the determinant of $\det J_f$ in \eqref{eqn:det J_f ReLU} indicates that a ReLU planar flow keeps a halfspace of $\mathbb{R}$ and applies linear scaling transformation to the other halfspace.

Given that the input distribution $q$ is Gaussian, we prove it is possible to approximate any density supported on a finite union of intervals in $\ell_1$ norm using a finite number of ReLU planar flows.

\begin{theorem}[Universal Approximation]
\label{thm: 1d ReLU universal approx}
Let $p$ be a density on $\mathbb{R}$ supported on a finite union of intervals. Then, for any $\epsilon>0$, there exists a flow $f$ composed of finitely many ReLU planar flows and a Gaussian distribution $q_{\mathcal{N}}$ such that $\|f\#q_{\mathcal{N}}-p\|_1\leq\epsilon$.
\end{theorem}

There are two steps in the proof. First, we show that Gaussian distributions can be exactly transformed to tail-consistent piecewise Gaussian distributions (see \textbf{Definition} \ref{def:PWG}, \textbf{Definition} \ref{def:tail} for formal definitions and \textbf{Lemma} \ref{lemma: 1d ReLU PWG exact}). An example of a tail-consistent piecewise Gaussian distribution of three pieces is shown in Figure \ref{fig: tail consistent pwg}: the distribution
is composed of three Gaussian pieces in full lines of three colors, where the dashed lines are corresponding prolongations. Then, the area below yellow lines (\textcolor{orange}{---}/\textcolor{orange}{- -}) is equal to the area below the blue dashed line (\textcolor{blue}{- -}), and the area below the green full line (\textcolor{green}{---}) is equal to the area below the yellow dashed line (\textcolor{orange}{- -}).

In the second step, we show that tail-consistent piecewise distributions can approximate any piecewise constant distribution supported on a finite union of compact intervals (see \textbf{Lemma} \ref{lemma: 1d ReLU approx PWC}). Notice that piecewise constant functions supported on a finite union of compact intervals can approximate any Lebesgue-integrable function \citep{lin2018resnet}, so do densities supported on a finite union of intervals. Therefore, the universal approximation property of ReLU planar flows (\textbf{Theorem} \ref{thm: 1d ReLU universal approx}) is obtained.

\begin{figure}[!h]
    \centering
    \includegraphics[trim=10 0 10 10, clip, width=0.45\textwidth]{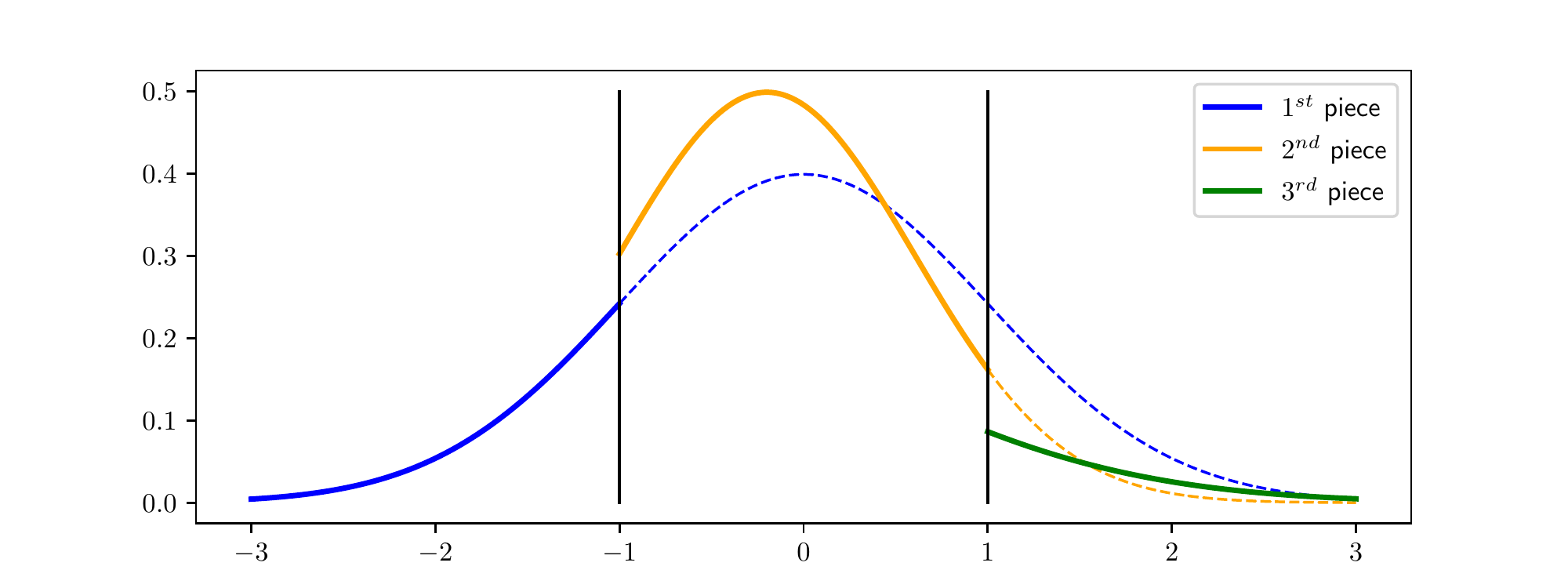}
    \caption{A tail-consistent piecewise Gaussian distribution in $\mathcal{PW}(3,\mathcal{G})$.}
    \label{fig: tail consistent pwg}
\end{figure}

In Figure \ref{fig: 1d relu}, two examples are presented on approximating the same target distribution $p$ with different number of ReLU planar flows. As illustrated, the approximation almost reaches perfection when we choose a larger number of ReLU planar flows.

\begin{figure}[!h]
    \centering
    \includegraphics[trim=10 0 10 10, clip, width=0.45\textwidth]{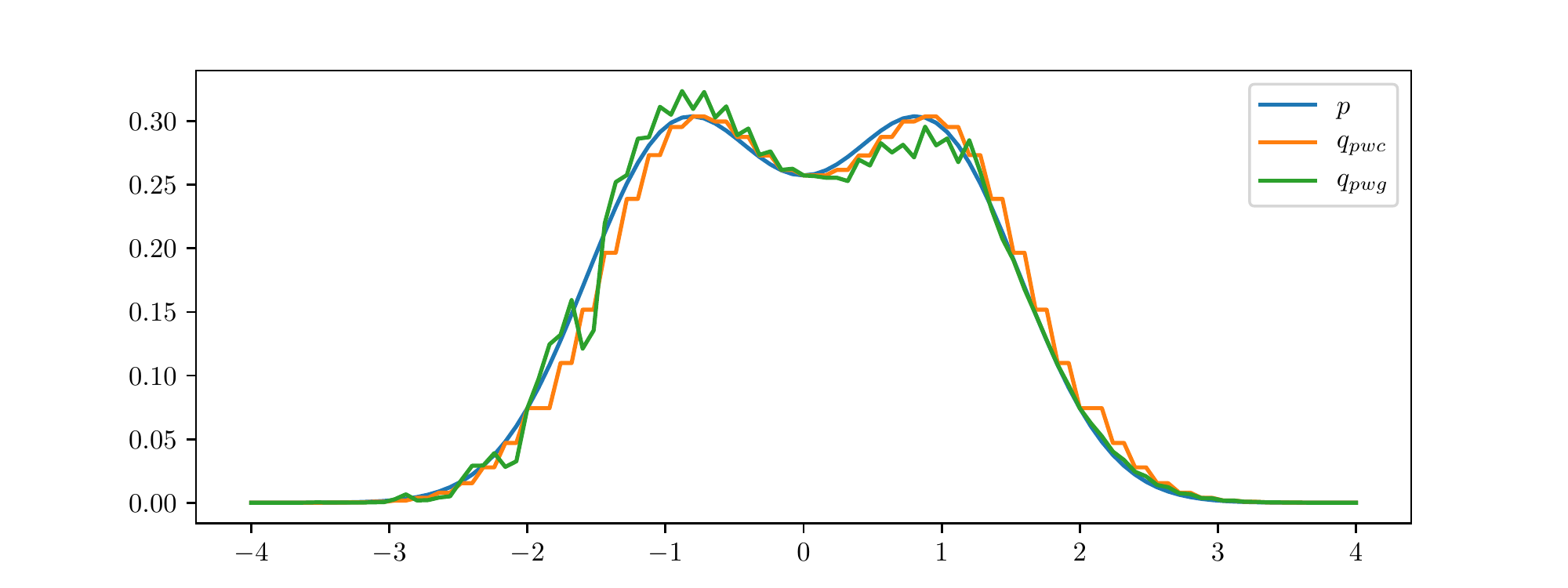}\\
    \includegraphics[trim=10 0 10 10, clip, width=0.45\textwidth]{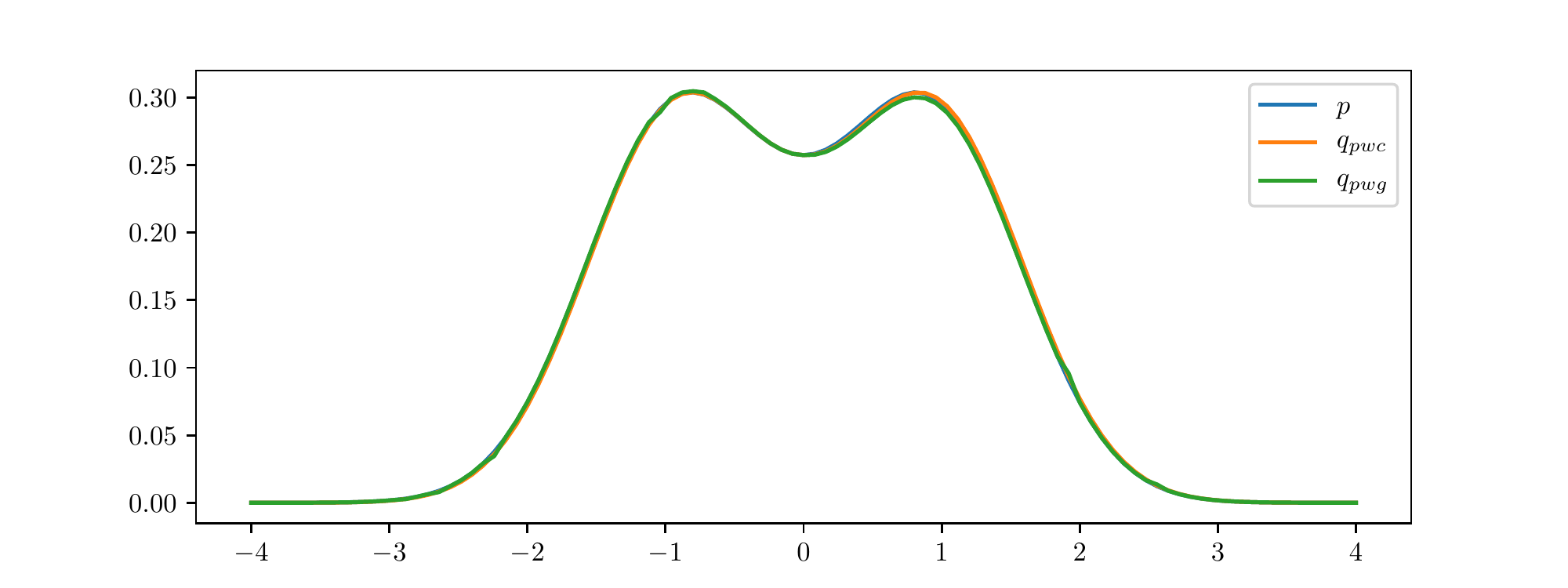}
    \caption{Target distribution $p$, its piecewise constant distribution approximation $q_{pwc}$ of 50 (top)/300 (bottom) pieces, and its tail-consistent piecewise Gaussian distribution approximation $q_{pwg}$ generated by 50 (top)/300 (bottom) ReLU planar flows over a Gaussian.}
    \label{fig: 1d relu}
\end{figure}

\begin{remark}
Since we can transform the standard Gaussian distribution $\mathcal{N}(0,1)$ to any other Gaussian distribution using a scaling function, which can be achieved by two ReLU planar flows and a shift, we can further assign the input distribution $q_{\mathcal{N}}$ in \textbf{Theorem} \ref{thm: 1d ReLU universal approx} to be the standard Gaussian distribution.
\end{remark}

\section{Exact Transformation for $d>1$}
\label{sec:high d exact}
In this section, we consider the exact transformation question when the data dimension $d>1$. We study two cases where the flow is composed of a finite number of Sylvester flows with ($i$) ReLU non-linearity and ($ii$) general non-linearity. We specifically show how the topology matching conditions yield negative results to the exact transformation question (that is, to show there does not exist such flow that can transform between certain distributions).

Our results are based on the following key observation for a flow $f:\mathbb{R}^d\rightarrow\mathbb{R}^d$. For almost every $z\in\mathbb{R}^d$ there exists a subspace $\mathcal{V}(z)\subset\mathbb{R}^d$ such that for any $v\in\mathcal{V}$ and small $\alpha>0$, $\det J_f(z)=\det J_f(z+\alpha v)$. We call $\mathcal{V}$ the complementary subspace of $f$ at $z$. This observation can be used to determine what class of distributions flows can transform between. By letting $\alpha\rightarrow0$, we can focus on properties of small neighbourhoods around $z$, which we call \textit{topology matching}.

\subsection{ReLU Non-linearity}
\label{sec:high d ReLU}
We begin with constructing a topology matching condition for ReLU Sylvester flows: $f(z)=f_{\mathrm{syl}}(z)=Z+A~\text{ReLU}(B^{\top}z+b)$. \eqref{eqn:det J_f h} shows that for a single ReLU Sylvester flow, if $B^{\top}z+b\neq0$, then $\det J_f(z')=\det J_f(z)$ when $z'$ is close to $z$. This statement can be further generalized: if $f$ is a flow composed of a finite number of ReLU Sylvester flows, for almost every $z\in\mathbb{R}^d$, the determinant of the Jacobian of $f$ is a constant near $z$. Based on this observation, we conclude that the complementary subspace $\mathcal{V}(z)=\mathbb{R}^d,\ a.e.$ (see \textbf{Lemma} \ref{lemma:ReLU local linear}). Using this property, we construct the topology matching condition in the following theorem.

\begin{theorem}[Topology Matching for ReLU Sylvester flows]
\label{thm:high d ReLU}
Suppose distribution $q$ is defined on $\mathbb{R}^d$, and flow $f$ is composed of finitely many ReLU Sylvester flows on $\mathbb{R}^d$. Let $p=f\#q$. Then, there exists a zero-measure closed set $\Omega\subset\mathbb{R}^d$ such that $\forall z\in\mathbb{R}^d\setminus\Omega$, we have
\[J_f(z)^{\top}\nabla_z\log p(f(z))=\nabla_z\log q(z)\]
\end{theorem}

Intuitively, the local directional derivatives of the logarithm of the density are preserved. As a special case, if $z$ satisfies $\nabla_z q(z)=0$ (which means that $z$ is a local minima, local maxima, or saddle point of $q$), then $p(f(z))$ must also have zero gradient at $z$. For instance, suppose $p$ is the standard Gaussian distribution on $\mathbb{R}^2$ and $q$ is a mixture of two Gaussian distributions on $\mathbb{R}^2$ with two peaks. Since only at the origin does $p$ have zero gradient, we conclude there does not exist a planar flow that transforms $q$ to $p$. Additional examples are illustrated in Figure \ref{fig: 2d relu topo} in the Appendix.

The proof of \textbf{Theorem} \ref{thm:high d ReLU} follows from \eqref{eqn:change of variable log}, the Taylor expansion of $f$, and the observation that $\mathcal{V}(z)=\mathbb{R}^d$ $a.e.$. Notably, the conclusion holds for any number of ReLU Sylvester flows. Using this condition, we show in the following corollaries that it is unlikely for finitely many ReLU Sylvester flows to transform between mixture of Gaussian (MoG) or product (Prod) distributions unless special conditions are satisfied.

\begin{corollary}[MoG$\nrightarrow$MoG]
\label{cor:MoG to MoG}
(See formal version in \textbf{Corollary} \ref{cor:MoG to MoG formal})
Suppose $p,q$ are mixture of Gaussian distributions on $\mathbb{R}^d$ in the following form:
\[p(z)=\sum_{i=1}^{r_p} w_p^i\mathcal{N}(z;\mu_p^i,\Sigma_p),\
  q(z)=\sum_{j=1}^{r_q} w_q^j\mathcal{N}(z;\mu_q^j,\Sigma_q)\]
Then, there generally does not exist flow $f$ composed of finitely many ReLU Sylvester flows such that $p=f\#q$.
\end{corollary}

\begin{corollary}[Prod$\nrightarrow$Prod]
\label{cor:Prod to Prod}
(See formal version in \textbf{Corollary} \ref{cor:Prod to Prod formal})
Suppose $p$ and $q$ are product distributions in the following form:
\[p(z)\propto\prod_{i=1}^d g(z_i)^{r_p};\
  q(z)\propto\prod_{i=1}^d g(z_i)^{r_q}\]
where $r_p,r_q>0, r_p\neq r_q$, and $g$ is a smooth function. Then, there generally does not exist flow $f$ composed of finitely many ReLU Sylvester flows such that $p=f\#q$.
\end{corollary}

Given our negative results, the reader might wonder what distributions can be transformed by ReLU Sylvester flows. We show that certain linear transformations can be exactly expressed (see \textbf{Theorem} \ref{thm: ReLU linear}, \textbf{Corollary} \ref{cor: linear transformation} and \textbf{Corollary} \ref{cor: N to N positive}).

\subsection{General Smooth Non-linearity}
\label{sec:high d h}
In this section, we construct a topology matching condition for Sylvester flows with general non-linearities. Suppose $f$ is a Sylvester flow $f(z)=z+Ah(B^{\top}z+b)$ with flow dimension $m$, where $h$ is an arbitrary smooth function. Analogous to \textbf{Theorem} \ref{thm:high d ReLU}, there exists a $d-m$ dimensional complementary subspace of $f$ at every point $z\in\mathbb{R}^d$: $\mathcal{V}(z)=\Span\{B\}^{\perp}$. Using this property, we are able to establish the topology matching condition for a single Sylvester flow (see \textbf{Lemma} \ref{lemma:high d exact 1}). Then, we generalize this result to $n$ layers of Sylvester flows in the following theorem.

\begin{theorem}[Topology Matching for Sylvester flows]
\label{thm:high d exact more}
Suppose distribution $q$ is defined on $\mathbb{R}^d$, and $n$ Sylvester flows $\{f_i\}_{i=1}^n$ on $\mathbb{R}^d$ have flow dimensions $\{m_i\}_{i=1}^n$, tangent matrices $\{B_i\}_{i=1}^n$, and smooth non-linearities. Let $f=f_n\circ\cdots\circ f_1$ and $p=f\#q$. Then $\forall z\in\mathbb{R}^d$, we have
\[\nabla_z\log p(f(z))-\nabla_z\log q(z)\in \Span\{B_1,B_2,\cdots,B_n\}\]
\end{theorem}

When the sum of flow dimensions of $\{f_i\}_{i=1}^n$ is strictly less than the data dimension $d$, $\Span\{B_1,B_2,\cdots,B_n\}$ is a strict subspace of $\mathbb{R}^d$. Under this situation, we show in the following corollary that transformation between Gaussian distributions might be impossible with a bounded number of Sylvester flows.

\begin{corollary}[$\mathcal{N}\nrightarrow\mathcal{N}$]
\label{cor:N to N more}
(See formal version in \textbf{Corollaries} \ref{cor:N to N 1 formal} and \ref{cor:N to N more formal})
Let $p\sim\mathcal{N}(0,\Sigma_p), q\sim\mathcal{N}(0,\Sigma_q)$ be two Gaussian distributions on $\mathbb{R}^d$, and $\Sigma_q^{-1}-\Sigma_p^{-1}$ has high rank. Then, with a limited number of planar or Sylvester flows that have smooth non-linearities, it is impossible to transform $q$ to $p$.
\end{corollary}

Additional experiments are demonstrated in Figure \ref{fig: 2d general topo} in the Appendix. We also construct a topology matching condition for radial flows in \textbf{Theorem} \ref{thm:high d radial}, and compare that result with \textbf{Theorem} \ref{thm:high d exact more}.

\section{Approximation Capacity for Large $d$}
\label{sec:high d approx}
In this section, we provide a partially negative answer to the universal approximation question for certain normalizing flows by showing that approximations in these cases may require very deep flows. In particular, we study local planar flows and Householder flows with specific target distributions.

Given an input distribution $q$ and a target distribution $p$ on $\mathbb{R}^d$, our goal is to lower bound the depth $T$ of a normalizing flow that can transform $q$ to an approximation of $p$. This is formally defined below.
\begin{definition}
\label{def: T}
Let $p,q$ be two distributions on $\mathbb{R}^d$, $\epsilon>0$, and $\mathcal{F}$ be a set of normalizing flows. Then, the minimum number of flows in $\mathcal{F}$ required to transform $q$ to an approximation of $p$ to within $\epsilon$ is
\[\begin{array}{rl}
    T_{\epsilon}(p,q,\mathcal{F})=\inf\{n:
    & \exists \{f_i\}_{i=1}^n\in\mathcal{F}\text{ such that } \\
    & \|(f_1\circ\cdots\circ f_n)\#q-p\|_1\leq\epsilon\}
\end{array}\]
\end{definition}

To achieve this goal, we look at the maximum $\ell_1$ norm distance reduction of a normalizing flow $f$ towards $p$:
\[\mathcal{L}(p,f)=\sup_{q'\text{ is a density on }\mathbb{R}^d}\|p-q'\|_1-\|p-f\#q'\|_1\]

We first show a surprisingly concise upper bound $\hat{\mathcal{L}}$ of $\mathcal{L}$. This bound is used in proving \textbf{Theorem} \ref{thm: ell_1 local planar} and \textbf{Theorem} \ref{thm: ell_1 householder} in this section.
\begin{lemma}
\label{lemma: upper bound of L}
$\mathcal{L}(p,f)
\leq\hat{\mathcal{L}}(p,f)$, where
\[\hat{\mathcal{L}}(p,f)
=\int_{\mathbb{R}^d}\left||\det J_f(z)|p(f(z))-p(z)\right|dz\]
\end{lemma}
Then, we naturally obtain a lower bound of $T$:
\[T_{\epsilon}(p,q,\mathcal{F})\geq\frac{\|p-q\|_1-\epsilon}{\sup_{f\in\mathcal{F}}\mathcal{L}(p,f)}\geq\frac{\|p-q\|_1-\epsilon}{\sup_{f\in\mathcal{F}}\hat{\mathcal{L}}(p,f)}\]

Next, we make the following assumption on $q$:
\begin{assumption}
\label{assump: random input}
$\|p-q\|_1=\Theta(1)$.
\end{assumption}
This assumption holds when the input distribution $q$ is a random initialization (that is, $q$ is chosen arbitrarily without any prior knowledge on $p$). Then, under \textbf{Assumption} \ref{assump: random input}, there exists $\epsilon>0$ (e.g. $\epsilon=\frac12\|p-q\|_1$) such that
\[T_{\epsilon}(p,q,\mathcal{F})
=\Omega\left(\frac{1}{\sup_{f\in\mathcal{F}}\hat{\mathcal{L}}(p,f)}\right)\]
In the rest of this section, we use this lower bound on $T$ to construct results for local planar flows and Householder flows with specific target distributions.

\subsection{Local Planar Flows}
In this section, we look at a specific group of planar flows, which we call the \textit{local} planar flows. A $c_h$-local planar flow is defined below.
\begin{definition}
A non-linearity $h$ is called $c_h$-local if there is a constant $c_h\in\mathbb{R}$ satisfying for any $x\in\mathbb{R}$, $(i)$ $|h(x)|\leq c_h$, and $(ii)$ $|h'(x)|\leq c_h/(1+|x|)$.
A planar flow $f(z)=z+uh(w^{\top}z+b)$ is called $c_h$-local if $h$ is $c_h$-local,  $\|u\|_2\leq1$, and $\|w\|_2\leq1$.
\end{definition}

Many popular non-linearities are $c_h$-local, such as $\tanh$ ($c_h=2$), $\text{sigmoid}$ ($c_h=1$), and $\arctan$ ($c_h=\pi/2$).

Geometrically, a local planar flow applies non-linear scaling on the region near the $d-1$ dimensional subspace $\{z:w^{\top}z+b=0\}$ in $\mathbb{R}^d$, while having little effect on regions far away from the subspace (almost a constant shift). This observation leads to the intuition that one layer of local planar flow can only affect a small volume of the whole space, so a large number of layers is needed to approximate the target distribution if $\supp p$ is a large region. In the following theorem, we show for certain $p$, $T$ goes up polynomially in the data dimension $d$ with adjustable degrees.

\begin{theorem}[$\ell_1$ norm approximation lower bound for local planar flows]
\label{thm: ell_1 local planar}
Let $p$ be a distribution on $\mathbb{R}^d$ $(d>2)$ such that for $\tau\in(0,1)$:
\begin{itemize}
    \item $p=\mathcal{O}(p_1)$, where density $p_1$ satisfies
    \[p_1(z)\propto\exp(-\|z\|_2^{\tau})\]

    \item $\|\nabla p\|_2=\mathcal{O}(\|\nabla p_2\|_2)$, where density $p_2$ satisfies
    \[p_2(z)\propto\left\{
        \begin{array}{cc}
            \exp(-d) & \|z\|_2\leq d^{\frac{1}{\tau}} \\
            \exp(-\|z\|_2^{\tau}) & \|z\|_2 > d^{\frac{1}{\tau}}
        \end{array}\right.\]
\end{itemize}
Suppose $\mathcal{F}$ is the set of all $c_h$-local planar flows.
Then, under \textbf{Assumption} \ref{assump: random input}, there exists $\epsilon=\Theta(1)$ such that
\[T_{\epsilon}(p,q,\mathcal{F})=\Omega\left(\min\left((\log d)^{-\frac{1}{\tau}}d^{\left(\frac{1}{\tau}-\frac12\right)}, d^{\left(\frac{1}{\tau}-1\right)}\right)\right)\]
\end{theorem}

This indicates that if the target distribution $p$ has specifically bounded values and gradients, a large number of local planar flows is needed to approximate $p$ starting with a distribution $q$ that obeys \textbf{Assumption} \ref{assump: random input}. The number $T$ is polynomial in $d$ with adjustable degrees, so it can be incredibly large as $d$ gets large.

A concrete example that satisfies the condition in \textbf{Theorem} \ref{thm: ell_1 local planar} is when $p(z)$ is equal to the $p_2$ in the statement. This satisfies the first condition because $\exp(-d)\leq\exp(-\|z\|_2^{\tau})$ in the ball centered at the origin with radius $d^{1/\tau}$, and the integration of $p_1$ in this ball is $o(1)$ (see proof of \textbf{Lemma} \ref{lemma: ell_1 1}). Then, taking for instance $\tau=0.2$, the lower bound on $T$ becomes $\Omega(d^4)$, which is incredibly large in practical scenarios.

To prove \textbf{Theorem} \ref{thm: ell_1 local planar}, we first show that $\hat{\mathcal{L}}(p,f)$ is upper bounded by an integration of two terms. We then present \textbf{Lemma} \ref{lemma: ell_1 1} and \textbf{Lemma} \ref{lemma: ell_1 2} to bound these two terms separately.

\subsection{Householder Flows}
In this section, we look at Householder flows. Since a Householder matrix does not change the $\ell_2$ norm of any vector, it is possible to upper bound $\mathcal{L}$ when the target distribution $p$ is almost symmetric, according to \textbf{Lemma} \ref{lemma: upper bound of L}. If $p$ is a standard Gaussian distribution, we have $\mathcal{L}=0$, indicating that Householder flows cannot transform any different distribution to a standard Gaussian distribution. In the following theorem, we provide a concise bound on $T$ when $p$ is very close to the standard Gaussian distribution, where there is only a small perturbation on its covariance matrix.

\begin{theorem}[$\ell_1$ norm approximation lower bound for Householder flows]
\label{thm: ell_1 householder}
Let $p$ be a Gaussian distribution $\mathcal{N}(0,I+S)$ on $\mathbb{R}^d$ $(d>2)$, where $|S_{ij}|\leq d^{-(2+\kappa)}$ for some $\kappa>0$ and any $1\leq i,j\leq d$. Suppose $\mathcal{F}$ is the set of all Householder flows. Then, under \textbf{Assumption} \ref{assump: random input}, there exists $\epsilon=\Theta(1)$ such that
\[T_{\epsilon}(p,q,\mathcal{F})=\Omega\left(d^{\kappa}\right)\]
\end{theorem}
This indicates that we need a large number of Householder flows to approximate a distribution close to the standard Gaussian distribution, starting with a distribution $q$ that obeys \textbf{Assumption} \ref{assump: random input}. The number $T$ is also polynomial in the data dimension $d$ with adjustable degrees, so it could be large as well. The bound is computed from $\hat{\mathcal{L}}$, where $|\det J_f(z)|=1$ for a Householder flow $f$.

\section{Additional Related Work}
\label{sec:morerelatedwork}
~~
\subsection{Normalizing Flows}
It was shown that transforming a simple distribution to a complicated one by composing many simple transformations can be used to solve density estimation problems \citep{tabak2010density, tabak2013family}. These transformations are called \textit{normalizing flows}. Two basic normalizing flows (planar and radial flows) were introduced \citep{rezende2015variational}. Due to their empirical success, there has been a growing body of work on other kinds of normalizing flows. Two categories of normalizing flows have been developed.

\textit{Triangular flows}.
It was proven that increasing triangular functions can transform between arbitrary distributions \citep{villani2008optimal}. Therefore, triangular flows composed of fixed classes of increasing triangular functions are expected to enjoy good expressive power. In addition, the determinant of the Jacobian matrix of an increasing triangular function is easy to compute. These two benefits have led to the development of a large family of triangular flows \citep{dinh2014nice, germain2015made, uria2016neural, kingma2016improved, dinh2016density, papamakarios2017masked, huang2018neural, jaini2019sum}. Among these flows, IAF \citep{kingma2016improved}, NAF \citep{huang2018neural} and SOS flows \citep{jaini2019sum} were shown to have the universal approximation property.

\textit{Non-triangular flows}. It is possible to calculate the determinant of the Jacobian matrix and the inverse of a well designed non-triangular function. Several flows parameterized by matrices were inspired by results from linear algebra and thus enjoy this property \citep{tomczak2016improving, Leonard2017orthogonal, ho2019flow++, berg2018sylvester}, where the last one is a matrix-form generalization of the planar flow. Moreover, a recent non-triangular flow, the iResNet \citep{behrmann2018invertible}, in the form of residual networks (ResNet) \citep{he2016deep}, was designed with an efficient $\log$-$\det$ approximator. It was further improved in residual flows with an unbiased approximator \citep{chen2019residual}. However, the expressivity of these flows still remain unknown, even though the iResNet is expressed by powerful neural networks.

\subsection{Continuous Time Flows}
It is possible, from the infinitesimal point of view, to generalize the discrete update of finite flows to continuous update of infinite flows. Infinite flows are described by a differential equation instead of a sequence of transformations in the finite flow context \citep{chen2017continuous, grathwohl2018ffjord, chen2018neural, salman2018deep, zhang2018monge}. The neural ODEs \citep{chen2018neural} is one significant work in this class, but its expressivity still lacks understanding. A counter-example was provided on the expressivity of the neural ODEs \citep{dupont2019augmented}.
However, this does not rigorously imply that neural ODEs are not universal approximators because $(i)$ the failure in exact transformation does not imply the impossibility in approximation, and $(ii)$ universal transformation does not necessarily need universal function representation.

To tackle the problem of such counter-example, additional $p$ dimensions were introduced to "augment" the neural ODEs \citep{dupont2019augmented}. By solving a $d+p$ dimensional augmented ODE and extracting the first $d$ dimensions, the expressivity of the neural ODEs is enhanced. It was further shown that the augmented neural ODEs is a universal approximator in the continuous function space when $p=1$ \citep{zhang2019approximation}. Nevertheless, in the context of normalizing flows, every transformation has to be invertible, so the change of dimension strategy, as well as its universal approximation property, does not apply to normalizing flows.

\section{Conclusions}
Normalizing flows are a class of deep generative models that offer flexible generative modeling as well as easy likelihood computation. While there has been a great deal of prior empirical work on different normalizing flow models, not much is (formally) known about their expressive power; we provide one of the first systematic studies on non-triangular flows. Our results demonstrate that one needs to be careful while designing normalizing flow models as well as their non-linearities in high dimensional space. In particular, we show that Sylvester flows, a universal approximator in one dimension, are unable to exactly transform between two (even simple) distributions unless rigorous conditions are satisfied. Additionally, a prohibitively large number of layers of planar or Householder flows are required to reduce the $\ell_1$ distance between input and output distributions under certain conditions.

There are a large number of open problems. Some unresolved problems towards expressivity of simple flows include
$(i)$ are certain combinations of tangent matrices or non-linearities useful,
$(ii)$ can normalizing flows composed of finitely many ($\geq d$) Sylvester flows with arbitrary non-linearities (or other simple flows) transform between any pair of input-output distributions in high dimensional space,
$(iii)$ are such normalizing flows universal approximators in converting distributions, and
$(iv)$ what class of distributions are easy or hard for normalizing flows composed of Sylvester flows or other simple flows to transform between. A final open problem is to look at other, more general classes of flows, and provide upper and lower bounds on their expressive power under different non-linearities.

\section*{Acknowledgements}
We thank NSF under IIS 1617157 for research support.

\bibliographystyle{apalike}
\bibliography{main}
\appendix
\onecolumn
\section{Appendix}

\subsection{Geometric Intuition of Planar and Radial FLows}
\label{sec: planad vs radial}
\begin{figure}[!h]
    \centering
    \includegraphics[trim=30 70 30 70,clip,width=0.45\textwidth]{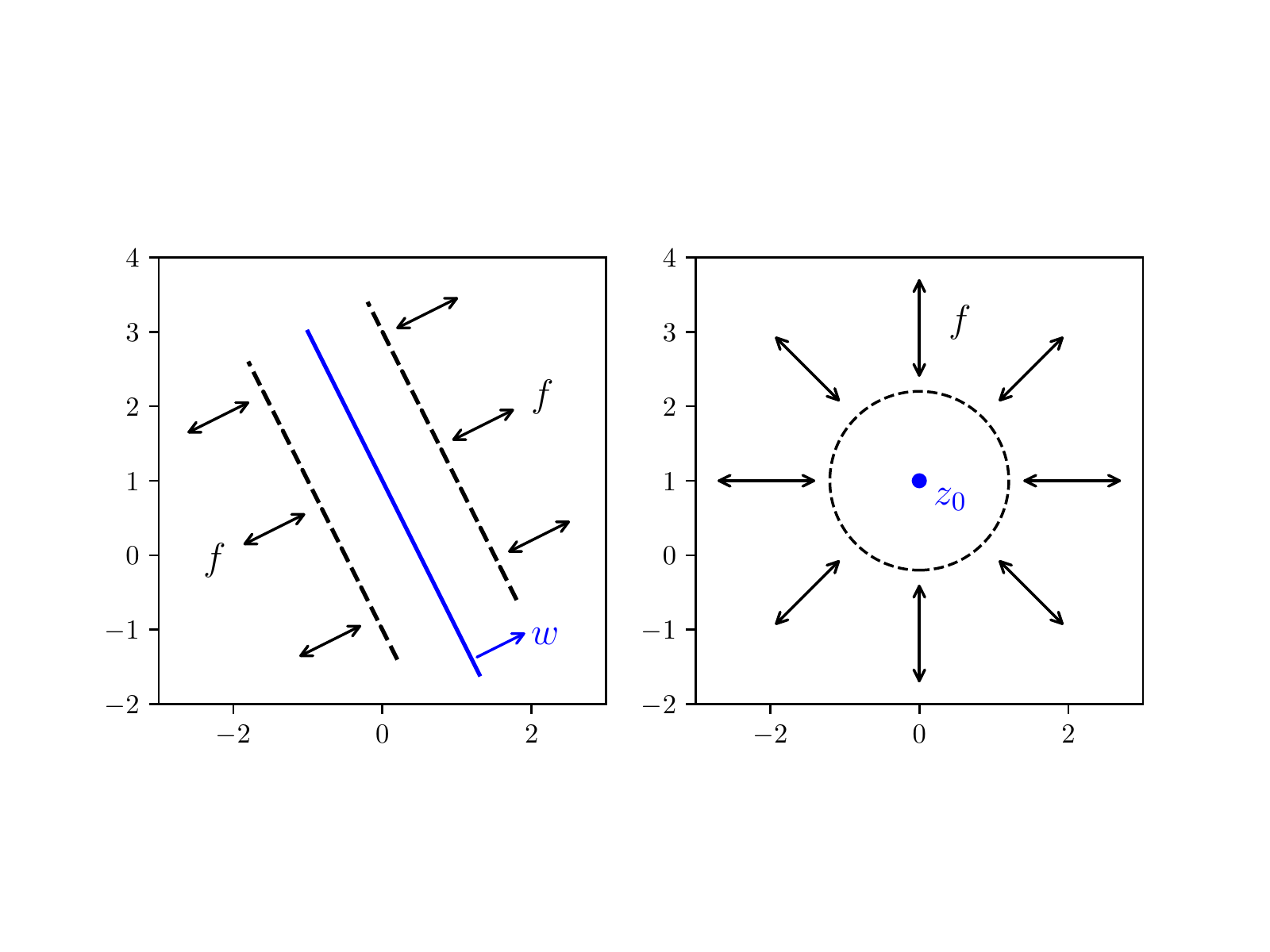}
    \caption{Geometric intuition of planar (left) versus radial (right) flows. In $\mathbb{R}^d$, a planar flow functions as a non-linear scaling transformation w.r.t. a $d-1$ dimensional subspace in the Cartesian coordinate system, while a radial flow functions as a non-linear scaling transformation w.r.t. center $z_0$ in the polar coordinate system. The bidirectional arrows mean that the scaling transformation is either expansion or compression.}
    \label{fig: planar vs radial}
\end{figure}

\subsection{Proof of Theorem \ref{thm:1d approx}}

\begin{definition}
$\Phi_p$ is defined as the cumulative function of distribution $p$:
\[\Phi_p(z)=\int_{-\infty}^{z_1}dx_1\int_{-\infty}^{z_2}dx_2\cdots\int_{-\infty}^{z_d}dx_d\ p(x)dx\]
\end{definition}

\begin{lemma}
[Possible Transformations (single flow)]
\label{lemma:1d general}
If $p$ and $q$ are densities on $\mathbb{R}$ supported on $n$ non-intersecting intervals:
\[\supp q=\bigcup_{i=1}^n \left(l_i^{(q)}, r_i^{(q)}\right),\
\supp p=\bigcup_{i=1}^n \left(l_i^{(p)}, r_i^{(p)}\right)\]
and if $\Phi_q\left(r_i^{(q)}\right)=\Phi_p\left(r_i^{(p)}\right)\ \forall 1\leq i\leq n$, then there exists a planar flow $f$ such that $f\#q=p$, $a.e.$.
\end{lemma}

\begin{proof}
As a special case of \textbf{Lemma} \ref{lemma:1d general}, if two densities $\tilde{p},q$ are supported on $\mathbb{R}$, we can transform $q$ into $\tilde{p}$ with a planar flow. Notice that for any density supported on a finite union of intervals, it is possible to approximate it using densities supported on a finite union of intervals excluding infinity. Therefore, we only need to prove for the following case:
\[\supp p=\bigcup_{i=1}^n \left(l_i, r_i\right)\]

To achieve this, it is sufficient to prove that there exists a distribution $\tilde{p}$ with support equal to $\mathbb{R}$ that can approximate $p$ to within $\epsilon$ for any $\epsilon>0$. We construct $\tilde{p}$ in the following way. We first define the threshold
\[\Delta=\frac{2}{\sum_{i=1}^n\left(r_i-l_i\right)}\]
Then, the measure of the set of points $x$ with density $p(x)\geq\Delta$ is at most $1/\Delta$, and thus the measure of the set of points $x$ with density $p(x)\in(0,\Delta)$ is at least $1/\Delta$. Define
\[\gamma=\int_{x:0<p(x)<\Delta}p(x)dx\leq1\]

Now, we define $\tilde{p}(x)$ to be:
\begin{itemize}
    \item If $p(x)\geq\Delta$, then $\tilde{p}(x)=p(x)$.
    \item If $0<p(x)<\Delta$, then $\tilde{p}(x)=(1-\epsilon/2)p(x)$.
    \item If $x\in \left[r_i, l_{i+1}\right]$ for some $i$, then
        \[\tilde{p}(x)=\frac{\epsilon\gamma}{2n\left(l_{i+1}-r_i\right)}\]
    \item If $x\leq l_1$ or $x\geq r_n$, then we assign $\tilde{p}(x)$ to be a tail of Gaussian distribution such that on this halfspace $\tilde{p}(x)\leq\epsilon/2$ and the integration of it is $\frac{\epsilon\gamma}{4n}$.
\end{itemize}

It can be examined that
\[\begin{array}{rl}
    \displaystyle \|\tilde{p}\|_1 &
    \displaystyle = \int_{p(x)\geq\Delta}|\tilde{p}(x)|dx
    +\int_{0<p(x)<\Delta}|\tilde{p}(x)|dx
    +\int_{p(x)=0}|\tilde{p}(x)|dx \\
    & \displaystyle = 1-\gamma+ (1-\epsilon/2)\gamma +\sum_{i=1}^{n-1} \frac{\epsilon\gamma(l_{i+1}-r_i)}{2n(l_{i+1}-r_i)} + \frac{\epsilon\gamma}{2n}\\
    & \displaystyle = 1
\end{array}\]
\[\begin{array}{rl}
    \displaystyle \|p-\tilde{p}\|_1 &
    \displaystyle = \int_{0<p(x)<\Delta}|p(x)-\tilde{p}(x)|dx
    +\int_{p(x)=0}|p(x)-\tilde{p}(x)|dx \\
    & \displaystyle = \frac{\epsilon\gamma}{2}+\sum_{i=1}^{n-1} \frac{\epsilon\gamma(l_{i+1}-r_i)}{2n(l_{i+1}-r_i)} + \frac{\epsilon\gamma}{2n}\\
    & \displaystyle = \epsilon\gamma\leq\epsilon
\end{array}\]
Thus we finish the proof.
\end{proof}

\subsection*{Proof of Lemma \ref{lemma:1d general}}
\begin{proof}
We construct such an $f(z)$ for $z$ in different regions, and then show that this $f$ can be written as a planar flow with a continuous non-linearity. To satisfy $f\#q=p$, $a.e.$, it is equivalent to show that $\Phi_p(f(z))=\Phi_q(z)$ for any $z\in\mathbb{R}$.
\begin{itemize}
    \item If $z\in \left(l_i^{(q)}, r_i^{(q)}\right)$ for some $i$, then $f(z)=\Phi_p^{-1}\circ \Phi_q(z)$. Since $q(z)>0$ in this interval, \[\Phi_p(l_i^{(p)})=\Phi_q(l_i^{(q)})<\Phi_q(z)<\Phi_q(r_i^{(q)})=\Phi_p(r_i^{(p)})\]
    Therefore, $\Phi_p^{-1}\circ \Phi_q(z)$ exists. Since $p,q$ are densities, $\Phi_p$ and $\Phi_q$ are continuous. Notice that $\Phi_p$ is increasing in a compact neighbourhood of $\Phi_q(z)$. Therefore, $\Phi_p^{-1}$ is continuous, so $f$ is continuous.

    \item If $z\in \left[r_i^{(q)}, l_{i+1}^{(q)}\right]$ for some $i$, we let
    \[f(z)=\frac{l_{i+1}^{(p)}-r_i^{(p)}}{l_{i+1}^{(q)}-r_i^{(q)}}\left(z-r_i^{(q)}\right)+r_i^{(p)}\]
    Intuitively, $f$ linearly maps $\left[r_i^{(q)},l_{i+1}^{(q)}\right]$ to $\left[r_i^{(p)},l_{i+1}^{(p)}\right]$. Then, we have if $z\in \left[r_i^{(q)}, l_{i+1}^{(q)}\right]$
    \[\Phi_p(f(z))=\Phi_p\left(r_i^{(p)}\right)=\Phi_q\left(r_i^{(q)}\right)=\Phi_q(z)\]
    To keep the continuity of $f$, we show that the boundary conditions are also satisfied:
    \[f\left(r_i^{(q)}\right)=r_i^{(p)}, f\left(l_{i+1}^{(q)}\right)=l_{i+1}^{(p)}\]

    \item If $z\geq r_n^{(q)}$, then $f(z)=z-r_n^{(q)}+r_n^{(p)}$ satisfies $\Phi_p(f(z))=\Phi_q(z)=1$ and $f$ is continuous. If $z\leq l_1^{(q)}$, then $f(z)=z-l_1^{(q)}+l_1^{(p)}$ satisfies $\Phi_p(f(z))=\Phi_q(z)=0$ and $f$ is continuous.
\end{itemize}
Now, we obtain an $f$ that is continuous on $\mathbb{R}$ and satisfies $f\#q=p$. Finally, if we set \[h(z)=\frac1u f\left(\frac{z-b}{w}\right)-\frac{z-b}{uw}\]
for any $u(\neq0), w(\neq0)$ and $ b$, then we can see that $f$ can be written as a planar flow: $f(z)=z+uh(wz+b)$.
\end{proof}

\subsection{Proof of Theorem \ref{thm: 1d ReLU universal approx}}

\begin{definition}[Piecewise Distributions in $\mathcal{C}$]
\label{def:PWC}
Let $\mathcal{C}_0$ be the set of distributions with continuous densities.
Suppose $\mathcal{C}\subset\mathcal{C}_0$, then we define $\mathcal{PW}(n,\mathcal{C})$ to be the set of all distributions $p$ on $\mathbb{R}$ satisfying: there exists real numbers $-\infty=t_0<t_1<\cdots<t_{n-1}<\infty$
such that for any $i=0,\cdots,n-1$, on the $i$-th interval ($(-\infty,t_1)$ if $i=0$, $[t_{n-1},\infty)$ if $i=n-1$, $[t_i,t_{i+1})$ otherwise) $p$ is equal to some distribution $p_i\in\mathcal{C}$. For conciseness, we say $p$ is described by $\{p_i,t_i\}_{i=0}^{n-1}$.
We define $\mathcal{PW}(n)=\mathcal{PW}(n,\mathcal{C}_0)$. If $n'>n$, then $\mathcal{PW}(n)\subset\mathcal{PW}(n')$.
\end{definition}

\begin{definition}[Piecewise Gaussian Distributions]
\label{def:PWG}
Let $\mathcal{G}$ be the set of Gaussian distributions $\{\mathcal{N}(\mu,\sigma^2): \mu\in\mathbb{R},\ \sigma>0\}$. We define the set of piecewise Gaussian distributions to be $\mathcal{PW}(n, \mathcal{G})$.
\end{definition}

\begin{definition}[Tail-consistency]
\label{def:tail}
Suppose $p\in\mathcal{PW}(n)$ is described by $\{p_i,t_i\}_{i=0}^{n-1}$. We say $p$ is tail-consistent w.r.t. $t_k$ if
\[\sum_{i=1}^k\int_{t_{i-1}}^{t_i}p_i(z)dz+\int_{t_k}^{\infty}p_{k+1}(z)dz=1\]
If $p$ is tail-consistent w.r.t. $t_k$ for any $k=1,\cdots,n-1$, we say $p$ is tail-consistent.
\end{definition}

\begin{lemma}[Possible Transformations (single flow)]
\label{lemma: 1d ReLU PWG exact single}
Let any two distributions $p,q\in\mathcal{PW}(n,\mathcal{G})$ satisfying: $p$ can be described by $\{p_i,t_i\}_{i=0}^{n-1}$ and $q$ can be described by $\{q_i,t_i\}_{i=0}^{n-1}$, where $p_i=q_i$ for $i<n-1$ (that is, the only difference is $p_{n-1}\neq q_{n-1}$). Then there exists a ReLU planar flow $f$ such that $f\#q=p$.
\end{lemma}

\begin{lemma}[Possible Transformations (flows)]
\label{lemma: 1d ReLU PWG exact}
$\forall p\in\mathcal{PW}(n,\mathcal{G})$, if $p$ is tail-consistent, then there exists $n-1$ ReLU planar flows $\{f_t\}_{t=1}^{n-1}$ and a Gaussian distribution $q_\mathcal{N}$ such that $(f_{n-1}\circ\cdots\circ f_1)\#q_\mathcal{N}=p$.
\end{lemma}

\begin{lemma}
\label{lemma: 1d ReLU approx PWC}
Given any piecewise constant distribution $q_{pwc}$ supported on a finite union of compact intervals, $\forall\epsilon>0$, there exists a tail-consistent piecewise Gaussian distribution $q_{pwg}$ such that $\|q_{pwg}-q_{pwc}\|_1\leq\epsilon$.
\end{lemma}

\begin{proof}
According to \citep{lin2018resnet}, piecewise constant functions supported on a finite union of compact intervals can approximate any Lebesgue-integrable function, so do densities supported on a finite union of intervals. Therefore, there exists such piecewise constant distribution $q_{pwc}$ such that $\|q_{pwc}-p\|_1\leq\epsilon/2$. According to \textbf{Lemma} \ref{lemma: 1d ReLU approx PWC}, there exists a tail-consistent piecewise Gaussian distribution $q_{pwg}$ such that $\|q_{pwg}-q_{pwc}\|_1\leq\epsilon/2$. According to \textbf{Lemma} \ref{lemma: 1d ReLU PWG exact}, there exists a flow $f$ composed of finitely many ReLU planar flows and a Gaussian distribution $q_{\mathcal{N}}$ such that $q_{pwg}=f\#q_{\mathcal{N}}$. As a result, we have $\|f\#q_{\mathcal{N}}-p\|_1\leq\epsilon$.
\end{proof}

\subsection*{Proof of Lemma \ref{lemma: 1d ReLU PWG exact single}}
\begin{proof}
By assumption, $p(y)=q(y)$ if $y<t_{n-1}$. Now, we assume on $[t_{n-1},\infty)$, $q\sim\mathcal{N}(\mu_n,\sigma_n^2)$, and $p\sim\mathcal{N}(\hat{\mu},\hat{\sigma}^2)$. Let $f$ be a ReLU planar flow with parameters $u,w$ and $b$, where $u=sgn(\hat{\sigma}-\sigma_n),w=|1-\hat{\sigma}/\sigma_n|, b=-wt_{n-1}$. Then, for any $y\in\mathbb{R}$,
\[f^{-1}(y)=\left\{
\begin{array}{cc}
    y & wy+b<0\\
    \frac{y-ub}{1+uw} & wy+b\geq0
\end{array}
\right.
=\left\{
\begin{array}{cc}
    y & y<t_{n-1}\\
    \frac{y-ub}{1+uw} & y\geq t_{n-1}
\end{array}
\right.\]
According to \eqref{eqn:change of variable} and \eqref{eqn:det J_f ReLU}, if $y< t_{n-1}$, $(f\#q)(y)=q(y)$. If $y\geq t_{n-1}$,
\[(f\#q)(y)=\frac{q\left(\frac{y-ub}{1+uw}\right)}{1+uw}
=\frac{\mathcal{N}\left(\frac{y-ub}{1+uw};\mu_n,\sigma_n^2\right)}{1+uw}\]
Thus, on $[t_{n-1},\infty)$,
\[f\#q\sim\mathcal{N}(ub+(1+uw)\mu_n,(1+uw)^2\sigma_n^2)
=\mathcal{N}((1+uw)\mu_n-uwt_{n-1},(1+uw)^2\sigma_n^2)\]
Since $uw=\frac{\hat{\sigma}}{\sigma_n}-1$, $f\#q\sim\mathcal{N}(\tilde{\mu},\hat{\sigma}^2)$ for some $\tilde{\mu}$ on $[t_{n-1},\infty)$. Notice that
\[\int_{t_{n-1}}^{\infty}\mathcal{N}(y;\tilde{\mu},\hat{\sigma}^2)dy
=1-\sum_{i=0}^{n-2}\int_{t_{i}}^{t_{i+1}}q_i(y)dy
=1-\sum_{i=0}^{n-2}\int_{t_{i}}^{t_{i+1}}p_i(y)dy
=\int_{t_{n-1}}^{\infty}\mathcal{N}(y;\hat{\mu},\hat{\sigma}^2)dy\]
we know that $\tilde{\mu}=\hat{\mu}$. Thus, the ReLU flow with the above $u,w$ and $b$ transforms the right-most piece of the input distribution $q$ to the desired target $p$ without changing the other pieces.
\end{proof}

\subsection*{Proof of Lemma \ref{lemma: 1d ReLU PWG exact}}
\begin{proof}
We prove by induction. For $n=1$, the result is obvious since any Gaussian distribution can be chosen as input. Suppose we are able to generate any tail-consistent distribution in $\mathcal{PW}(n-1,\mathcal{G})$. Given the target distribution $p\in\mathcal{PW}(n,\mathcal{G})$ described by $\{p_i,t_i\}_{i=0}^{n-1}$, where
\[p_i(z)=\mathcal{N}(z;\mu_i,\sigma_i^2),\ i=0,\cdots,n-1\]
we first generate an intermediate distribution $q_{int}\in\mathcal{PW}(n-1,\mathcal{G})$ described by $\{q_i,t_i\}_{i=0}^{n-2}$, where
\[q_i=p_i,\ i=0,\cdots,n-2\]
Since $p$ is tail-consistent, $q_{int}$ integrates to 1 on $\mathbb{R}$, so it is a probability distribution. Notice that $q_{int}$ can be viewed as an element in $\mathcal{PW}(n,\mathcal{G})$ described by $\{q_i,t_i\}_{i=0}^{n-1}$, where $q_{n-1}=q_{n-2}$.
Then, according to \textbf{Lemma} \ref{lemma: 1d ReLU PWG exact single}, we can apply one more layer of ReLU flow to transform $q_{int}$ into the desired distribution $p$.
\end{proof}

\subsection*{Proof of Lemma \ref{lemma: 1d ReLU approx PWC}}
\begin{proof}
Suppose the target distribution $q_{pwc}$ has a compact support $\subset[t_-,t_+]$, where $q_{pwc}(t_-)$ and $q_{pwc}(t_+)$ are strictly positive. We construct $q_{pwg}$ as follows. First , we let
\[\int_{-\infty}^{t_-}q_{pwg}(x)dx=\int_{t_+}^{\infty}q_{pwg}(x)dx=\frac{\epsilon}{3}\]
This can be done by setting $q_{pwg}=\mathcal{N}(\mu_-,\sigma_-^2)$ on $(-\infty, t_-)$ where $(t_--\mu_-)/\sigma_-=\Phi_{\mathcal{N}(0,1)}^{-1}(\epsilon/3)$, and $q_{pwg}=\mathcal{N}(\mu_+,\sigma_+^2)$ on $(t_+,\infty)$ where $(t_+-\mu_+)/\sigma_+=\Phi_{\mathcal{N}(0,1)}^{-1}(1-\epsilon/3)$.

On $[t_-,t_+]$, suppose $q_{pwc}$ is a piecewise constant function on $n$ intervals of $\delta_i$ width, where $\delta/2<\delta_i<\delta$ for $1\leq i\leq n$, and $\delta$ is an arbitrarily small positive value. Then, the number of intervals $n$ is $\Theta(1/\delta)$.

Now, we look at the $i$-th interval, where $1\leq i\leq n$. Suppose $q_{pwc}(x)=\alpha$ for $x\in[t,t+\delta_i)$. Then, a valid tail-consistent piecewise Gaussian piece on this interval has the form $\mathcal{N}(\mu,\sigma^2)$ with
\[\int_{t}^{\infty}\mathcal{N}(x;\mu,\sigma^2)dx=\left(1-\frac23\epsilon\right)\int_{t}^{\infty}q_{pwc}(x)dx\]
This guarantees that $q_{pwg}$ is tail-consistent and integrates to 1 on $\mathbb{R}$. The solution of $\mu$ and $\sigma$ is given by $(t-\mu)/\sigma=c$ for some constant $c$ such that $|c|\leq\Phi_{\mathcal{N}(0,1)}^{-1}(\epsilon/3)$. Now, we show that $\mathcal{N}(x;\mu,\sigma^2)$ approximates $\alpha$ in $\ell_1$ norm on $[t,t+\delta_i)$. If $\alpha=0$, by letting $\sigma\rightarrow\infty$ we are able to approximate 0 to within any precision. Thus, we only discuss cases where $\alpha>0$. We assign $\mathcal{N}(t;\mu,\sigma^2)=\alpha$. The solution is given by
\[\sigma=\frac{\exp(-\frac{c^2}{2})}{\sqrt{2\pi}\alpha},\ \mu=t-c\sigma\]
One can check that the Lipschitz constant of the Gaussian distribution is $\frac{1}{\sqrt{2\pi e}\sigma^2}$. Thus, the $\ell_1$ norm of the difference between $\mathcal{N}(\mu,\sigma^2)$ and $\alpha$ on $[t,t+\delta_i)$ is bounded by
\[\int_{t}^{t+\delta_i}\left|\mathcal{N}(x;\mu,\sigma^2)-\alpha\right|dx
\leq\frac{\delta^2}{2\sqrt{2\pi e}\sigma^2}
=\sqrt{\frac{\pi}{2}}\alpha^2\exp\left(c^2-\frac12\right)\delta^2\]
Since we have finite subdivisions, $\alpha$ can be seen as an $\mathcal{O}(1)$ constant. Combining with the bound on $c$, we have
\[\int_{t}^{t+\delta_i}\left|\mathcal{N}(x;\mu,\sigma^2)-q_{pwc}(x)\right|dx\leq \sqrt{\frac{\pi}{2}}\left(\sup_{x\in\mathbb{R}}q_{pwc}(x)^2\right)\exp\left(\Phi_{\mathcal{N}(0,1)}^{-1}(\epsilon/3)^2-\frac12\right)\delta^2\]
Since there are $n\leq 2/\delta$ intervals, we know that
\[\int_{t_-}^{t_+}\left|\mathcal{N}(x;\mu,\sigma^2)-q_{pwc}(x)\right|dx\leq \sqrt{2\pi}\left(\sup_{x\in\mathbb{R}}q_{pwc}(x)^2\right)\exp\left(\Phi_{\mathcal{N}(0,1)}^{-1}(\epsilon/3)^2-\frac12\right)\delta\]
Since $\delta$ is arbitrary, we can assign
\[\delta=\frac{\epsilon}{3\sqrt{2\pi}\left(\sup_{x\in\mathbb{R}}q_{pwc}(x)^2\right)\exp\left(\Phi_{\mathcal{N}(0,1)}^{-1}(\epsilon/3)^2-\frac12\right)}\]
Then,
\[\int_{-\infty}^{\infty}|q_{pwg}(x)-q_{pwc}(x)|dx=\left(\int_{-\infty}^{t_-}+\int_{t_-}^{t_+}+\int_{t_+}^{\infty}\right)|q_{pwg}(x)-q_{pwc}(x)|dx\leq\frac{\epsilon}{3}+\frac{\epsilon}{3}+\frac{\epsilon}{3}=\epsilon\]
As a result, $\forall \epsilon>0$, there exists a tail-consistent distribution $q_{pwg}\in\mathcal{PW}(n+2,\mathcal{G})$ satisfying $\|q_{pwg}-q_{pwc}\|_1\leq\epsilon$, where $n=\mathcal{O}\left(\exp\left(\Phi_{\mathcal{N}(0,1)}^{-1}(\epsilon/3)^2\right)/\epsilon\right)$.
\end{proof}

\subsection{Proof of Theorem \ref{thm:high d ReLU}}

\begin{lemma}
\label{lemma:ReLU local linear}
Let $\{f_i\}_{i=1}^n$ be $n$ ReLU Sylvester flows on $\mathbb{R}^d$ and $f=f_n\circ\cdots\circ f_1$. Then, there exists a zero-measure closed set $\Omega\subset\mathbb{R}^d$ such that $\forall x\in\mathbb{R}^d\setminus\Omega$, there exists an open neighbourhood of $x$ called $\Gamma_x$, such that $J_f(z)$ is equal to a constant matrix for $z\in\Gamma_x$.
\end{lemma}

\begin{proof}
According to \textbf{Lemma} \ref{lemma:ReLU local linear}, there exists a zero-measure closed set $\Omega\subset\mathbb{R}^d$ such that $\forall z\in\mathbb{R}^d\setminus\Omega$, $J_f(z)$ is constant in an open neighbourhood of $z$. By the change-of-variable formula in \eqref{eqn:change of variable log}, for small $\alpha\in\mathbb{R}$ and any direction $\delta\in\mathbb{R}^d$,
\[\log p(f(z+\alpha\delta))-\log q(z+\alpha\delta)
=\log p(f(z))-\log q(z)\]
Next, we expand the Taylor series of $f(z+\alpha\delta)$ for small $\alpha$:
\[f(z+\alpha\delta) = f(z) + \alpha J_f(z)\delta + \mathcal{O}(\alpha^2)\]
Therefore,
\[\log p(f(z) + \alpha J_f(z)\delta + \mathcal{O}(\alpha^2))-\log p(f(z))
=\log q(z+\alpha\delta)-\log q(z)\]
By multiplying $1/\alpha$ on both sides and taking $\alpha\rightarrow0$, we finish the proof.
\end{proof}

\begin{remark}
\textbf{Theorem} \ref{thm:high d ReLU} can be extended to any Sylvester flow with $h''=0$ almost everywhere.
\end{remark}

\begin{remark}
\textbf{Theorem} \ref{thm:high d ReLU} can be extended to Householder flows \citep{tomczak2016improving}.
\end{remark}

\begin{remark}
An example of directional derivative is illustrated in Figure \ref{fig: directional derivative}.

\begin{figure}[!h]
    \centering
    \includegraphics[trim=10 10 10 10, clip, width=0.4\textwidth]{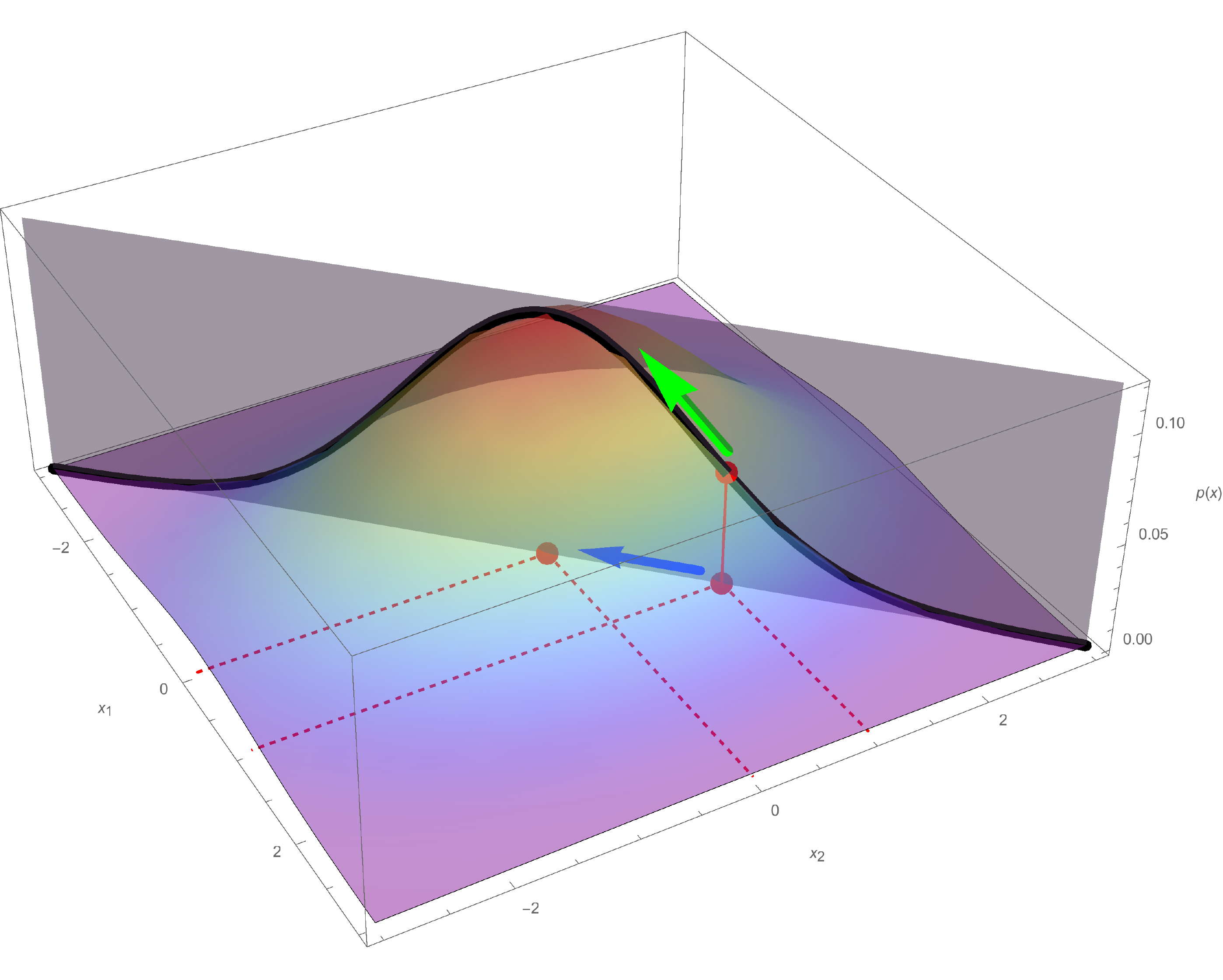}
    \caption{Directional derivative (green arrow) of a two-dimensional Gaussian distribution at point $z=(1,1)$ and direction $\delta=(-1,-1)$ (blue arrow).}
    \label{fig: directional derivative}
\end{figure}
\end{remark}

\subsection*{Proof of Lemma \ref{lemma:ReLU local linear}}

\begin{proof}
Suppose the $i$th Sylvester flow $f_i$ has parameters $A_i,B_i,b_i$ for $i=1,\cdots,n$. Notice that when $B_i^{\top}z+b_i\neq0$, there exists an open set $\mathcal{B}_z$ containing $z$ such that $\forall y\in\mathcal{B}_z$, the signs of $B_i^{\top}y+b_i$ is identical to those of $B_i^{\top}z+b_i$. Therefore, $J_{f_i}(y)$ is equal to a constant matrix in $\mathcal{B}_z$. Then, the statement straightly follows from the chain rule of Jacobian matrix, where
\[\Omega=\bigcup_{i=1}^n \{z:B_i^{\top}z+b_i=0\}\]
\end{proof}

\subsection{Formal Version of Corollary \ref{cor:MoG to MoG}}

\begin{corollary}[MoG$\nrightarrow$MoG]
\label{cor:MoG to MoG formal}
Suppose $p,q$ are mixture of Gaussian distributions on $\mathbb{R}^d$ in the following form:
\[p(z)=\sum_{i=1}^{r_p} w_p^i\mathcal{N}(z;\mu_p^i,\Sigma_p),\
  q(z)=\sum_{j=1}^{r_q} w_q^j\mathcal{N}(z;\mu_q^j,\Sigma_q)\]
If a flow $f$ composed of finitely many ReLU Sylvester flows satisfies $p=f\#q$, then for almost every point $x\in\mathbb{R}^d$, it has an open neighbourhood $\Gamma_x$ such that $\forall z\in \Gamma_x$,
\[\begin{array}{rl}
&\displaystyle \frac{\sum_{i,j=1}^{r_p}w_p^i\mathcal{N}(Az+b;\mu_p^i,\Sigma_p)\mathcal{N}(Az+b;\mu_p^j,\Sigma_p)A^{\top}\Sigma_p^{-1}\mu_p^i(\mu_p^i-\mu_p^j)^{\top}\Sigma_p^{-1}A}
{\left(\sum_{j=1}^{r_p}w_p^j\mathcal{N}(Az+b;\mu_p^j,\Sigma_p)\right)^2} \\
-&\displaystyle \frac{\sum_{i,j=1}^{r_q}w_q^i\mathcal{N}(z;\mu_q^i,\Sigma_q)\mathcal{N}(z;\mu_q^j,\Sigma_q)\Sigma_q^{-1}\mu_q^i(\mu_q^i-\mu_q^j)^{\top}\Sigma_q^{-1}}
{\left(\sum_{j=1}^{r_q}w_q^j\mathcal{N}(z;\mu_q^j,\Sigma_q)\right)^2}
\end{array}\]
is a constant function in $z$ on $\Gamma_x$ for some $A\in\mathbb{R}^{d\times d}$ and $b\in\mathbb{R}^d$.
\end{corollary}

\begin{proof}
Suppose $f=f_1\circ\cdots\circ f_n$ is a normalizing flow composed of finite ReLU Sylvester flows. For almost every $x\in\mathbb{R}^d$, we have $J_f$ is equal to a constant matrix $A$ in an open neighbourhood of $x$ called $\Gamma_x$. That is, for some $b\in\mathbb{R}^d$,
\[f(z)=Az+b,\ \forall z\in\Gamma_x\]
Now, we solve the topology matching condition in \textbf{Theorem} \ref{thm:high d ReLU} on $\Gamma_x$.
\[\begin{array}{rl}
    \displaystyle \nabla_z\log q(z)
    &\displaystyle  = -\frac{1}{q(z)}\Sigma_q^{-1}\left(\sum_{j=1}^{r_q}w_q^j\mathcal{N}(z;\mu_q^j,\Sigma_q)(z-\mu_q^j)\right)\\
    &\displaystyle  = -\Sigma_q^{-1}\left(z-\sum_{j=1}^{r_q}\frac{w_q^j\mathcal{N}(z;\mu_q^j,\Sigma_q)}{q(z)}\mu_q^j\right)
\end{array}\]
Similarly,
\[\begin{array}{rl}
    \displaystyle J_f(z)^{\top}\nabla_z\log p(f(z))
    &\displaystyle  = -\frac{1}{p(f(z))}J_f(z)^{\top}\Sigma_p^{-1}\left(\sum_{i=1}^{r_p}w_p^i\mathcal{N}(f(z);\mu_p^i,\Sigma_p)(f(z)-\mu_p^i)\right) \\
    &\displaystyle  = -J_f(z)^{\top}\Sigma_p^{-1}\left(f(z)-\sum_{i=1}^{r_p}\frac{w_p^i\mathcal{N}(f(z);\mu_p^i,\Sigma_p)}{p(f(z))}\mu_p^i\right)
\end{array}\]
Therefore, we obtain
\[\left(A^{\top}\Sigma_p^{-1}A-\Sigma_q^{-1}\right)z+A^{\top}\Sigma_p^{-1}b
=A^{\top}\Sigma_p^{-1}
\left(\sum_{i=1}^{r_p}\frac{w_p^i\mathcal{N}(Az+b;\mu_p^i,\Sigma_p)}{p(Az+b)}\mu_p^i\right)
-\Sigma_q^{-1}\left(\sum_{j=1}^{r_q}\frac{w_q^j\mathcal{N}(z;\mu_q^j,\Sigma_q)}{q(z)}\mu_q^j\right)\]
Notice that the left-hand-side is linear in $z$. Thus, if $p=f\#q$, then the right-hand-side is should be linear in $z$. By standard arithmetic we can calculate the derivative of the right-hand-side over $z$ as follow:
\[\begin{array}{rl}
&\displaystyle \frac{\sum_{i,j=1}^{r_p}w_p^i\mathcal{N}(Az+b;\mu_p^i,\Sigma_p)\mathcal{N}(Az+b;\mu_p^j,\Sigma_p)A^{\top}\Sigma_p^{-1}\mu_p^i(\mu_p^i-\mu_p^j)^{\top}\Sigma_p^{-1}A}
{\left(\sum_{j=1}^{r_p}w_p^j\mathcal{N}(Az+b;\mu_p^j,\Sigma_p)\right)^2} \\
-&\displaystyle \frac{\sum_{i,j=1}^{r_q}w_q^i\mathcal{N}(z;\mu_q^i,\Sigma_q)\mathcal{N}(z;\mu_q^j,\Sigma_q)\Sigma_q^{-1}\mu_q^i(\mu_q^i-\mu_q^j)^{\top}\Sigma_q^{-1}}
{\left(\sum_{j=1}^{r_q}w_q^j\mathcal{N}(z;\mu_q^j,\Sigma_q)\right)^2}
\end{array}\]
However, this is generally a non-constant function in $z$ except for some special cases.
\end{proof}

\begin{remark}
To give a simple case where the condition in \textbf{Corollary} \ref{cor:MoG to MoG formal} does not hold, we let $r_p=r_q=2$, $\mu_p^1=\mu_p^1$, $\mu_q^1\neq\mu_q^2$ and $w_q^1=w_q^2=\frac12$. Then, the difference in the condition is given by
\[-\frac{2\mathcal{N}(z;\mu_q^1,\Sigma_q)\mathcal{N}(z;\mu_q^2,\Sigma_q)}{\left(\mathcal{N}(z;\mu_q^1,\Sigma_q)+\mathcal{N}(z;\mu_q^2,\Sigma_q)\right)^2}\Sigma_q^{-1}(\mu_q^1-\mu_q^2)(\mu_q^1-\mu_q^2)^{\top}\Sigma_q^{-1}\]
If it is a constant function in $z$, then both $\mathcal{N}(z;\mu_q^1,\Sigma_q)+\mathcal{N}(z;\mu_q^2,\Sigma_q)$ and $\mathcal{N}(z;\mu_q^1,\Sigma_q)-\mathcal{N}(z;\mu_q^2,\Sigma_q)$ are equal to a constant times $\sqrt{\mathcal{N}(z;\mu_q^1,\Sigma_q)\mathcal{N}(z;\mu_q^2,\Sigma_q)}$. As a result, $\mathcal{N}(z;\mu_q^1,\Sigma_q)/\mathcal{N}(z;\mu_q^2,\Sigma_q)$ is a constant for $z\in\Gamma_x$. By expanding the density expression, we have $z^{\top}\Sigma_q^{-1}(\mu_q^1-\mu_q^2)$ is a constant for $z\in\Gamma_x$. However, since $\mu_q^1\neq\mu_q^2$, $\Sigma_q^{-1}(\mu_q^1-\mu_q^2)\neq 0$. Contradiction.
\end{remark}

\subsection{Formal Version of Corollary \ref{cor:Prod to Prod}}

\begin{corollary}[Prod$\nrightarrow$Prod]
\label{cor:Prod to Prod formal}
Suppose $p,q$ are product distributions in the following form:
\[p(z)\propto\prod_{i=1}^d g(z_i)^{r_p};\
  q(z)\propto\prod_{i=1}^d g(z_i)^{r_q}\]
where $r_p,r_q>0, r_p\neq r_q$, and $g$ is a smooth function. If a flow $f$ composed of finitely many ReLU Sylvester flows satisfies $p=f\#q$, then for almost every point $x\in\mathbb{R}^d$, it has an open neighbourhood $\Gamma_x$ such that $\forall z\in \Gamma_x$,
\[r_q\tilde{\nabla}\log\textbf{g}(z)=r_pA^{\top}\tilde{\nabla}\log\textbf{g}(Az+b)\]
holds for some $b\in\mathbb{R}^d$, where $\textbf{g}(z)=(g(z_1),\cdots,g(z_d))^{\top}$, and $\tilde{\nabla}$ takes the gradient of the $i$-th function w.r.t the $i$-th variable for $1\leq i\leq d$.
\end{corollary}

\begin{proof}
Suppose $f=f_1\circ\cdots\circ f_n$ is a normalizing flow composed of finite ReLU Sylvester flows. For almost every $x\in\mathbb{R}^d$, we have $J_f$ is equal to a constant matrix $A$ in an open neighbourhood of $x$ called $\Gamma_x$. That is, for some $b\in\mathbb{R}^d$,
\[f(z)=Az+b,\ \forall z\in\Gamma_x\]
Now, we solve the topology matching condition in \textbf{Theorem} \ref{thm:high d ReLU} on $\Gamma_x$. By matching the corresponding elements, we have the following result:
\[r_p\sum_{i=1}^{d}A_{ij}\frac{g'((Az+b)_i)}{g((Az+b)_i)}=r_q\frac{g'(z_j)}{g(z_j)},\ j=1,\cdots,d\]
Rewriting this equation into vector form, we finish our proof.
\end{proof}

\begin{remark}
To give a simple case where the condition in \textbf{Corollary} \ref{cor:Prod to Prod formal} does not hold, we let $d=2$ and $g(x)=x$. Then, the necessary condition becomes
\[\left\{\begin{array}{rl}
    \displaystyle \frac{r_q}{r_p z_1}
     &\displaystyle = \frac{A_{11}}{A_{11}z_1+A_{12}z_2+b_1}
      +\frac{A_{21}}{A_{21}z_1+A_{22}z_2+b_2} \\
    \displaystyle \frac{r_q}{r_p z_2}
     &\displaystyle = \frac{A_{12}}{A_{11}z_1+A_{12}z_2+b_1}
      +\frac{A_{22}}{A_{21}z_1+A_{22}z_2+b_2}
\end{array}\right.\]
or equivalently,
\[\begin{array}{rl}
    & \displaystyle r_q (A_{11}z_1+A_{12}z_2+b_1)(A_{21}z_1+A_{22}z_2+b_2) \\
    = &\displaystyle (A_{11}(A_{21}z_1+A_{22}z_2+b_2)+A_{21}(A_{11}z_1+A_{12}z_2+b_1))r_p z_1 \\
    = &\displaystyle (A_{12}(A_{21}z_1+A_{22}z_2+b_2)+A_{22}(A_{11}z_1+A_{12}z_2+b_1))r_p z_2 \\
\end{array}\]
By checking the $z_1 z_2$ term, we obtain $A_{11}A_{22}+A_{12}A_{21}=0$, which indicates that $\det A=0$. This contradicts the fact that $f$ is an invertible flow. As a result, there does not exist a flow composed of finitely many ReLU flows that transform $p(z)\propto (z_1z_2)^{r_p}$ to $q(z)\propto (z_1z_2)^{r_q}$.
\end{remark}

\subsection{Positive Results for ReLU Planar Flows}
\begin{theorem}[Linear Transformations]
\label{thm: ReLU linear}
If $A\in\mathbb{R}^d$ has the LU decomposition, then the linear transformation $g(z)=Az$ can be generated by $4d-4$ ReLU planar flows.
\end{theorem}
\begin{proof}
First, we show that certain rank-one-modification transformations ($f(z)=(I+R)z$ where $rank(R)=1$) can be achieved by composing two ReLU planar flows. Suppose $R=uw^{\top}$ where $\det (I+R)=1+u^{\top}w>0$. We assign
\[f_1(z)=z+h( w^{\top}z)u\]
\[f_2(z)=z-h(-w^{\top}z)u\]
then $f=f_2\circ f_1$: if $w^{\top}z<0$, then $f_1(z)=z$, so $f_2\circ f_1(z)=f_2(z)=z+uw^{\top}z$; if $w^{\top}z\geq0$, then $f_1(z)=z+uw^{\top}z$, and since $w^{\top}(I+uw^{\top})z=(1+u^{\top}w)w^{\top}z\geq0$, we have $f_2\circ f_1(z)=f_1(z)=z+uw^{\top}z$.

Now, assume that $A$ has the LU decomposition:
\[A=LU\]
where $L(U)$ is a lower(upper) triangular matrix. Notice that both $L$ and $U$ can be decomposed to a product of $d-1$ Frobenius matrices. Since the determinant of a Frobenius matrix is $1>0$, both $L$ and $U$ can be decomposed to product of $2(d-1)$ ReLU planar flows. Therefore, we need $4d-4$ planar flows to express $A$.
\end{proof}

\begin{corollary}
\label{cor: linear transformation}
For any $A\in\mathbb{R}^d$, the linear transformation $g(z)=Az$ can be generated by $4d-4$ ReLU planar flows and $d$ Householder flows.
\end{corollary}
\begin{proof}
Since any matrix has $LUP$ decomposition, we have $A=LUP$ where $L(U)$ is a lower(upper) triangular matrix, and $P$ is a permutation matrix. Since any permutation matrix is an orthogonal matrix, $P$ can be decomposed to a product of $d$ Householder matrices. Using the analysis in the proof of \textbf{Theorem} \ref{thm: ReLU linear}, we finish the proof.
\end{proof}

\begin{corollary}
\label{cor: N to N positive}
Given any Gaussian distributions $q\sim\mathcal{N}(0,\Sigma_q)$ and $p\sim\mathcal{N}(0,\Sigma_p)$ centered at the origin, we can transform $q$ into $p$ with $4d-4$ ReLU planar flows and $d$ Householder flows.
\end{corollary}
\begin{proof}
Notice that a PSD matrix $\Sigma$ can be decomposed to $Q^{\top}\Lambda Q$, where $Q$ is an orthogonal matrix and $\Lambda$ is a diagonal matrix. Therefore, we have
\[\Sigma_q=Q_q^{\top}\Lambda_q Q_q\]
\[\Sigma_p=Q_p^{\top}\Lambda_p Q_p\]
Now, we assign
\[f(z)=Q_p^{-1}\Lambda_p^{-\frac12}\Lambda_q^{\frac12}Q_qz\]
One can check that this linear function $f$ transforms $q$ into $p$. Using the result in \textbf{Corollary} \ref{cor: linear transformation}, we finish the proof.
\end{proof}

\subsection{Proof of Theorem \ref{thm:high d exact more}}

\begin{lemma}[Topology Matching for single Sylvester flow]
\label{lemma:high d exact 1}
Suppose distribution $q$ is defined on $\mathbb{R}^d$, and a Sylvester flow $f$ on $\mathbb{R}^d$ has tangent matrix $B$ and smooth non-linearity. Let $p=f\#q$. Then $\forall z\in\mathbb{R}^d$, we have
\[\nabla_z\log p(f(z))-\nabla_z\log q(z)\in \Span\{B\}\]
\end{lemma}

\begin{proof}
We prove by induction on $n$. If $n=1$, then it is equivalent to \textbf{Lemma} \ref{lemma:high d exact 1}. Suppose the conclusion holds for $n-1$: $\forall z\in\mathbb{R}^d$,
\[\nabla_z\log (g\#q)(g(z))-\nabla_z\log q(z)\in \Span\{B_1,\cdots,B_{n-1}\}\]
where $g=f_{n-1}\circ\cdots\circ f_1$. Then, we apply \textbf{Lemma} \ref{lemma:high d exact 1} on $f_n$ at $g(z)$. As a result, we obtain that $\forall z\in\mathbb{R}^d$,
\[\nabla_z\log ((f_n\circ g)\#q)(f_n\circ g(z))-\nabla_z\log (g\#q)(g(z))\in \Span\{B_n\}\]
By adding these two equations, we finish the proof.
\end{proof}

\subsection*{Proof of Lemma \ref{lemma:high d exact 1}}
\begin{proof}
For any $\alpha\in\mathbb{R}$, according to the expression of Sylvester flows, we have for any $w^{\perp}\in \Span\{B\}^{\perp}$,
\[B^{\top}w^{\perp}=\textbf{0}\]
Therefore,
\[f(z+\alpha w^{\perp})=z+\alpha w^{\perp}+Ah(B^{\top}z+b+\alpha B^{\top}w^{\perp})=f(z)+\alpha w^{\perp}\]
Therefore, $\det J_f(z)=\det J_f(z+\alpha w^{\perp})$. According to \eqref{eqn:change of variable log}, we have
\[\begin{array}{rcl}
    \log p(f(z)) & = & \log(q(z)) - \log\det J_f(z) \\
    \log p(f(z+\alpha w^{\perp})) & = & \log(q(z+\alpha w^{\perp})) - \log\det J_f(z+\alpha w^{\perp})
\end{array}\]
Subtracting these two equations, we have
\[\log p(f(z)+\alpha w^{\perp})-\log p(f(z)) = \log(q(z+\alpha w^{\perp})) - \log q(z)\]
By multiplying $1/\alpha$ on both sides and taking $\alpha\rightarrow0$, we have $\forall w^{\perp}\in \Span\{B\}^{\perp}$,
\[(\nabla_z\log p(f(z)))^{\top}w^{\perp}-(\nabla_z\log q(z))^{\top}w^{\perp}=0\]
Therefore, $\nabla_z\log p(f(z))-\nabla_z\log q(z)\in \Span\{B\}$.
\end{proof}

\begin{remark}

The property $f(z+\alpha w^{\perp})=f(z)+\alpha w^{\perp}$ is enjoyed exclusively by Sylvester flows. Let $g(z)=f(z)-z$, then we have $g(z+\alpha w^{\perp})=g(z)\ \forall z\in\mathbb{R}^d,\forall\alpha\in\mathbb{R},\forall w^{\perp}\in \Span\{B\}^{\perp}$. Therefore,
\[g(z)=g(\mathcal{P}_B z)\]
where $\mathcal{P}_B$ is the projection matrix to the subspace spanned by column vectors of $B$. Then, we have
\[f(z)=z+g(z)=z+g(\mathcal{P}_B z)\]
As a result, $f$ can be expressed as a Sylvester flow.
\end{remark}

\subsection{Formal Version of Corollary \ref{cor:N to N more} for Planar and Sylvester Flows}

\begin{corollary}[Planar flow $\mathcal{N}\nrightarrow\mathcal{N}$]
\label{cor:N to N 1 formal}
Let $p\sim\mathcal{N}(0,\Sigma_p), q\sim\mathcal{N}(0,\Sigma_q)$ be two Gaussian distributions on $\mathbb{R}^d$. If there exists a planar flow $f$ on $\mathbb{R}^d$ with smooth non-linearity such that $p=f\#q$, then $rank\left(\Sigma_q-\Sigma_p\right)\leq1$.
\end{corollary}

\begin{proof}
If there exists a planar flow $f(z)=z+uh(w^{\top}z+b)$ transforming $q$ into $p$, then according to \textbf{Lemma} \ref{lemma:high d exact 1}, we have $\forall z\in\mathbb{R}^d,\forall w^{\perp}\in \Span\{w\}^{\perp}$,
\[z^{\top}\Sigma_q^{-1}w^{\perp}=f(z)^{\top}\Sigma_p^{-1}w^{\perp}\]
or equivalently,
\[z^{\top}\left(\Sigma_q^{-1}-\Sigma_p^{-1}\right)w^{\perp}=h(w^{\top}z+b)u^{\top}\Sigma_p^{-1}w^{\perp}\]

First, by setting $z=0$, we obtain
\[h(b)u^{\top}\Sigma_p^{-1}w^{\perp}=0,\ \forall w^{\perp}\in \Span\{w\}^{\perp}\]
Then, by setting $z=w^{\perp}$ and using the above equation, we obtain
\[(w^{\perp})^{\top}\left(\Sigma_q^{-1}-\Sigma_p^{-1}\right)w^{\perp}=h(b)u^{\top}\Sigma_p^{-1}w^{\perp}=0,\ \forall w^{\perp}\in \Span\{w\}^{\perp}\]

\begin{itemize}
    \item If $w^{\top}\left(\Sigma_q^{-1}-\Sigma_p^{-1}\right)w=0$, then $\Sigma_p=\Sigma_q$.

    \item If $w^{\top}\left(\Sigma_q^{-1}-\Sigma_p^{-1}\right)w>0$, then $\Sigma_q^{-1}-\Sigma_p^{-1}$ is PSD, and can be factorized as $Q^{\top}\Lambda Q$, where $Q$ is orthogonal and $\Lambda$ is diagonal. As a result,
    \[\Lambda^{\frac12}Qw^{\perp}=0,\ \forall w^{\perp}\in \Span\{w\}^{\perp}\]
    This indicates that $rank\left(\Lambda^{\frac12}\right)=1$, or $rank\left(\Sigma_q^{-1}-\Sigma_p^{-1}\right)=1$.

    \item If $w^{\top}\left(\Sigma_q^{-1}-\Sigma_p^{-1}\right)w<0$, we do the same analysis to $\Sigma_p^{-1}-\Sigma_q^{-1}$ and obtain the same result as above.
\end{itemize}

Therefore, if $rank(\Sigma_q^{-1}-\Sigma_p^{-1})>1$, there does not exist such planar flow that transforms $q$ into $p$. Suppose $rank\left(\Sigma_q^{-1}-\Sigma_p^{-1}\right)=1$, Since covariance matrices are symmetric, we have $\Sigma_q^{-1}-\Sigma_p^{-1}=\pm\tilde{v}\tilde{v}^{\top}$. Therefore, $\Sigma_q=(\Sigma_p^{-1}\pm\tilde{v}\tilde{v}^{\top})^{-1
}$ for some $\tilde{v}\in\mathbb{R}^d$. According to the Sherman$-$Morrison formula \citep{sherman1950adjustment}, we obtain
\[\Sigma_q=\Sigma_p-\frac{\pm\Sigma_p\tilde{v}\tilde{v}^{\top}\Sigma_p}{1\pm\tilde{v}^{\top}\Sigma_p\tilde{v}}\]
By assigning $v=\frac{\Sigma_p\tilde{v}}{\sqrt{1\pm\tilde{v}^{\top}\Sigma_p\tilde{v}}}$, we obtain $\Sigma_p-\Sigma_q=\pm vv^{\top}$.
\end{proof}

\begin{corollary}[Sylvester flow $\mathcal{N}\nrightarrow\mathcal{N}$]
\label{cor:N to N more formal}
Let $p\sim\mathcal{N}(0,\Sigma_p), q\sim\mathcal{N}(0,\Sigma_q)$ be two Gaussian distributions on $\mathbb{R}^d$, and $A=\Sigma_q^{-1}-\Sigma_p^{-1}$ with eigenvalues $\lambda_1\leq\cdots\leq\lambda_d$. Suppose a flow $f$ on $\mathbb{R}^d$ composed of $n$ Sylvester flows with flow dimensions $\{m_i\}_{i=1}^n$ and smooth non-linearities satisfies $p=f\#q$. If $m=\sum_{i=1}^n m_i<d$, then we have $\lambda_{m+1}\geq0$, $\lambda_{d-m}\leq0$. As a result, $rank(A)\leq 2m$.
\end{corollary}

\begin{proof}
Since $m<d$, $U^*=\Span\{B_1,\cdots,B_n\}^{\perp}$ is a subspace of $\mathbb{R}^d$ with dimension at least $d-m$. According to the proof of \textbf{Corollary} \ref{cor:N to N 1 formal}, we have
\[(w^{\perp})^{\top}\left(\Sigma_q^{-1}-\Sigma_p^{-1}\right)w^{\perp}=0,\
\forall w^{\perp}\in U^*\]
Let $A=\Sigma_q^{-1}-\Sigma_p^{-1}$ with eigenvalues $\lambda_1\leq\cdots\leq\lambda_d$. According to the Courant-Fischer theorem (Chapter 5.2.2. (4), \citep{lutkepohl1996handbook}),
\[\begin{array}{rl}
    \lambda_{d-m} &
    \displaystyle \leq \lambda_{\dim{U^*}}\\
    & \displaystyle =\min_{\dim{W}=\dim{U^*}}\max_{x\in W,x\neq0} \frac{x^{\top}Ax}{x^{\top}x} \\
     & \displaystyle \leq \max_{x\in U^*,x\neq0} \frac{x^{\top}Ax}{x^{\top}x} \\
     & = 0
\end{array}\]
\[\begin{array}{rl}
    \lambda_{m+1} & \displaystyle \geq \lambda_{d+1-\dim{U^*}}\\
    & \displaystyle =\max_{\dim{W}=\dim{U^*}}\min_{x\in W,x\neq0} \frac{x^{\top}Ax}{x^{\top}x} \\
     & \displaystyle \geq \min_{x\in U^*,x\neq0} \frac{x^{\top}Ax}{x^{\top}x} \\
     & = 0
\end{array}\]
When $m+1\leq d-m$ (or $m<d/2$), we can infer that $\lambda_i=0$ for $m+1\leq i\leq d-m$. Therefore, $A$ has at least $d-2m$ zero eigenvalues. This indicates that $rank(A)\leq 2m$.
\end{proof}

\subsection{Comparison with Radial Flows}
In this section, we present the connection and difference between Sylvester and radial flows from geometric insights. First, we present the topology matching condition for a single radial flow in the following theorem.

\begin{theorem}[Topology Matching for single radial flow]
\label{thm:high d radial}
Suppose distribution $q$ is defined on $\mathbb{R}^d$, and a radial flow $f$ on $\mathbb{R}^d$ has smoothing factor $a\in\mathbb{R}^+$, scaling factor $b\in\mathbb{R}$, and center $z_0\in\mathbb{R}^d$. Let $p=f\#q$. Then $\forall z\in\mathbb{R}^d\setminus\{z_0\}$, we have
\[\left(1+\frac{b}{a+\|z-z_0\|_2}\right)\nabla_z\log p(f(z))-\nabla_z\log q(z)\]
is parallel to $z-z_0$.
\end{theorem}

Though similar to the condition presented in \textbf{Lemma} \ref{lemma:high d exact 1} in the high level sketch, there are two notable differences in \textbf{Theorem} \ref{thm:high d radial}: $(i)$ there is the additional term $\left(1+\frac{b}{a+\|z-z_0\|_2}\right)$ in the condition, and $(ii)$ the complementary subspace $\mathcal{V}$ for planar flows is invariant in $z$, while for radial flows $\mathcal{V}(z)=\Span\{z-z_0\}^{\perp}$ is dependent on $z$. Next, we show that a radial flow cannot transform between Gaussian distributions with different covariance matrices, an even stronger result than \textbf{Corollary} \ref{cor:N to N 1 formal}.

\begin{corollary}[$\mathcal{N}\nrightarrow\mathcal{N}$]
\label{cor:N to N 1 radial}
Let $p\sim\mathcal{N}(0,\Sigma_p), q\sim\mathcal{N}(0,\Sigma_q)$ be two Gaussian distributions on $\mathbb{R}^d$. If there exists a radial flow $f$ on $\mathbb{R}^d$ such that $p=f\#q$, then $\Sigma_q=\Sigma_p$.
\end{corollary}

\subsection*{Proof of Theorem \ref{thm:high d radial}}
\begin{proof}
By standard algebra, it can be shown that the Jacobian of $f$ is given by
\[J_f(z)=\left(1+\frac{b}{a\|z-z_0\|_2}\right)I-\frac{b (z-z_0)(z-z_0)^{\top}}{(a+\|z-z_0\|_2)^2\|z-z_0\|_2}\]
Therefore, its determinant is
\[\det J_f(z)=\left(1+\frac{b}{a+\|z-z_0\|_2}\right)^d-\left(1+\frac{b}{a+\|z-z_0\|_2}\right)^{d-1}\frac{b \|z-z_0\|_2}{(a+\|z-z_0\|_2)^2}\]
Notice that if $\|z-z_0\|_2$ does not change then $\det J_f(z)$ remains the same. Therefore, for any $z\neq z_0$, any direction $w^{\perp}\in \Span\{z-z_0\}^{\perp}$ and small positive real number $r$, we have
\[\det J_f(z+r w^{\perp})-\det J_f(z)=\mathcal{O}(r^2)\]
\[f(z+r w^{\perp})-f(z)=\left(1+\frac{b}{a+\|z-z_0\|_2}\right)r w^{\perp}+\mathcal{O}(r^2)\]
By the change-of-variable formula in \eqref{eqn:change of variable}
\[p(f(z))=\frac{q(z)}{|\det J_f(z)|},\ p(f(z+r w^{\perp}))=\frac{q(z+r w^{\perp})}{|\det J_f(z+r w^{\perp})|}\]
For $r$ small, $q(z)$ is continuous and positive in $B_r(z)$, so $p(f(z+rw^{\perp}))/q(z)=\mathcal{O}(1)$. Therefore,
\[\frac{q(z+rw^{\perp})}{q(z)}=\frac{p(f(z+rw^{\perp}))}{p(f(z))}+\mathcal{O}(r^2)\]
By taking the logarithm, multiplying $1/r$ , and letting $r\rightarrow0$ on both sides, we have that
\[\left(1+\frac{b}{a+\|z-z_0\|_2}\right)(\nabla_z\log p(f(z)))^{\top}w^{\perp}=(\nabla_z\log q(z))^{\top}w^{\perp}\]
Therefore,
\[\left(1+\frac{b}{a+\|z-z_0\|_2}\right)\nabla_z\log p(f(z))-\nabla_z\log q(z)\]
is parallel to $z-z_0$.
\end{proof}

\subsection*{Proof of Corollary \ref{cor:N to N 1 radial}}
\begin{proof}
Let the radial flow be $f(z)=z+\frac{b}{a+\|z-z_0\|_2}(z-z_0)$ with $b\neq0$. For conciseness, we write $v_x=1+\frac{b}{a+\|x\|_2}$ for any $x\in\mathbb{R}^d$. Now, we assign $x=z-z_0$ and solve the topology matching condition in \textbf{Theorem} \ref{thm:high d radial}. By standard algebra, we obtain for any $x\in\mathbb{R}^d\setminus\{0\}$, if $x^{\top}w^{\perp}=0$, then
\[\left(v_x(z_0+v_x x)^{\top}\Sigma_p^{-1}-(x+z_0)^{\top}\Sigma_q^{-1}\right)w^{\perp}=0\]
This indicates that
\[(v_x^2\Sigma_p^{-1}-\Sigma_q^{-1})x+(v_x\Sigma_p^{-1}-\Sigma_q^{-1})z_0\]
is parallel to $x$ (or equal to 0). By applying the same analysis to $-x$, we have
\[(v_x^2\Sigma_p^{-1}-\Sigma_q^{-1})(-x)+(v_x\Sigma_p^{-1}-\Sigma_q^{-1})z_0\]
is parallel to $x$ (or equal to 0). Adding these two vectors, we have $(v_x\Sigma_p^{-1}-\Sigma_q^{-1})z_0$ is parallel to $x$ (or equal to 0) for any $x\in\mathbb{R}^d\setminus\{0\}$. As a result, $z_0$ is the origin, and $(v_x^2\Sigma_p^{-1}-\Sigma_q^{-1})x$ is parallel to $x$. The only possibility to this claim is that $v_x^2\Sigma_p^{-1}-\Sigma_q^{-1}$ is a multiple of the identity matrix for any $x\in\mathbb{R}^d\setminus\{0\}$. Since $v_x$ varies as $x$ changes, both $\Sigma_p$ and $\Sigma_q$ are multiple of the identity matrix: $\Sigma_p=\kappa_p I,\ \Sigma_q=\kappa_q I$.

Next, we apply the results above to the change-of-variable equation in \eqref{eqn:change of variable log} of radial flow. By standard algebra, we have for any $z\in\mathbb{R}^d\setminus\{0\}$,
\[\frac12\log\kappa_p+\frac{v_z^2z^{\top}z}{2\kappa_p}=\frac12\log\kappa_q+\frac{z^{\top}z}{2\kappa_q}+\log|\det J_f(z)|\]
Notice that as $\|z\|_2\rightarrow\infty$, the left-hand-side is equal to $\frac{\|z\|_2^2}{2\kappa_p}+\frac{b}{\kappa_p}\|z\|_2+o(\|z\|_2)$, while the right-hand-side is equal to $\frac{\|z\|_2^2}{2\kappa_q}+\mathcal{O}(1)$. Then, $b$ must be $0$, which means that $f$ is the identity map, and $p,q$ are identical.
\end{proof}

\subsection{Proof of Lemma \ref{lemma: upper bound of L}}
\begin{proof}
According to \eqref{eqn:change of variable}, for any distribution $q'$ on $\mathbb{R}^d$, we have $(f\#q')(f(z))=q'(z)/|\det J_{f}(z)|$. By letting $y=f(z), z=f^{-1}(y)$, we have
\[
\begin{array}{rl}
    \displaystyle \int_{\mathbb{R}^d}|p(y)-(f\#q')(y)|dy
    & \displaystyle =\int_{\mathbb{R}^d}\left|p(y)-\frac{q'(z)}{|\det J_{f}(z)|}\right|dy\\
    & \displaystyle =\int_{\mathbb{R}^d}\left|p(f(z))-\frac{q'(z)}{|\det J_{f}(z)|}\right| |\det J_{f}(z)|dz\\
    & \displaystyle =\int_{\mathbb{R}^d}\left||\det J_{f}(z)|p(f(z))-q'(z)\right|dz
\end{array}\]
By the triangular inequality, we have
\[\int_{\mathbb{R}^d}|p(z)-q'(z)|dz-\int_{\mathbb{R}^d}\left||\det J_{f}(z)|p(f(z))-q'(z)\right|dz
\leq \int_{\mathbb{R}^d}\left||\det J_{f}(z)|p(f(z))-p(z)\right|dz\]
By taking the supremum over $q'$, we finish the proof.
\end{proof}

\subsection{Proof of Theorem \ref{thm: ell_1 local planar}}
\begin{lemma}
\label{lemma: ell_1 1}
Let $f(z)=z+uh(w^{\top}z+b)$ be a $c_h$-local planar flow. If $p(z)\propto\exp(-\|z\|_2^{\tau})$, then
\[\int_{\mathbb{R}^d}\frac{\|w\|_2p(z)}{1+|w^{\top}z+b|} dz=\mathcal{O}\left((\log d)^{\frac{1}{\tau}}d^{-\left(\frac{1}{\tau}-\frac12\right)}\right)\]
\end{lemma}

\begin{lemma}
\label{lemma: ell_1 2}
Let $f(z)=z+uh(w^{\top}z+b)$ be a $c_h$-local planar flow. If
\[p(z)\propto\left\{
    \begin{array}{cc}
        \exp(-d) & \|z\|_2\leq d^{\frac{1}{\tau}} \\
        \exp(-\|z\|_2^{\tau}) & \|z\|_2 > d^{\frac{1}{\tau}}
    \end{array}\right.\]
then
\[\int_{\mathbb{R}^d}\left(\Delta_{c_h}p\vert_z\right) dz=\mathcal{O}\left(d^{-\left(\frac{1}{\tau}-1\right)}\right)\]
\end{lemma}

\begin{proof}
Let $f(z)=z+uh(w^{\top}z+b)$ be a $c_h$-local planar flow. According to \textbf{Lemma} \ref{lemma: upper bound of L} and the fact that $|u^{\top}w h'(w^{\top}z+b))|<1$, we have
\[\begin{array}{rl}
\mathcal{L}(p,f)\leq\hat{\mathcal{L}}(p,f) & \displaystyle = \int_{\mathbb{R}^d}\left||\det J_f(z)|p(f(z))-p(z)\right|dz \\
& \displaystyle = \int_{\mathbb{R}^d}\left|(1+u^{\top}w h'(w^{\top}z+b))p(z+uh(w^{\top}z+b))-p(z)\right|dz
\end{array}\]
Now we define
\[\Delta_s p\vert_z=\sup_{\|\delta\|\leq s}|p(z+\delta)-p(z)|\]
Then, since $\forall x\in\mathbb{R}, |h(x)|\leq c_h, |h'(x)|\leq c_h/(1+|x|)$, we have
\[\begin{array}{rl}
    \hat{\mathcal{L}}(p,f)
    & \displaystyle = \int_{\mathbb{R}^d}\left|(1+u^{\top}w h'(w^{\top}z+b))(p(z+uh(w^{\top}z+b))-p(z))+u^{\top}w h'(w^{\top}z+b)p(z)\right|dz \\
    (|u^{\top}wh'|\leq c_h)
    &\displaystyle \leq \int_{\mathbb{R}^d}\left(|u^{\top}w h'(w^{\top}z+b)|p(z)+(1+c_h)|p(z+uh(w^{\top}z+b))-p(z)|\right) dz \\
    (\text{bounds on }h\text{ and }h')
    &\displaystyle \leq  \int_{\mathbb{R}^d}\left(\frac{c_h}{1+|w^{\top}z+b|}|u^{\top}w|p(z)+(1+c_h)\Delta_{\|u\|_2c_h}p\vert_z\right) dz\\
    (\|u\|_2\leq1)
    &\displaystyle \leq c_h\int_{\mathbb{R}^d}\frac{\|w\|_2p(z)}{1+|w^{\top}z+b|}dz+(1+c_h)\int_{\mathbb{R}^d}(\Delta_{c_h}p\vert_z) dz
\end{array}\]
Finally, using \textbf{Lemma} \ref{lemma: ell_1 1} and \textbf{Lemma} \ref{lemma: ell_1 2} and setting $\epsilon=\frac12\|p-q\|_1=\Theta(1)$, we finish the proof.
\end{proof}

\subsection*{Proof of Lemma \ref{lemma: ell_1 1}}
\begin{proof}
For conciseness, we denote $p(z)$ by $p(r)$ for any $z$ such that $\|z\|_2=r\geq0$. That is, $p(r)\propto \exp(-r^{\tau})$. The outline of the proof is: $(i)$ simplify the expression by showing $b=0$, $(ii)$ rewrite the expression in polar coordination system, and $(iii)$ apply bounds from Gamma functions to obtain the result.

\textbf{Step} 1: \textit{simplification by solving $w$ and $b$.} First of all, we have $\|w\|_2/(1+|w^{\top}z+b|)=1/(1/\|w\|_2+|\tilde{w}^{\top}z+b|)$, where $\tilde{w}=w/\|w\|_2$. Thus, to make the integration largest, $\|w\|_2$ should equal to 1. Since $p(z)$ is symmetric, any direction of $w$ yields the same result. Therefore, we set $w=e_d=(0,\cdots,0,1)^{\top}\in\mathbb{R}^d$.

Next, we show $b=0$. To maximize the integration, we have
\[\frac{\partial}{\partial b}\int_{\mathbb{R}^d}\frac{p(z)}{1+|z_d+b|} dz=0\]
The partial derivative is equal to
\[\begin{array}{rl}
\displaystyle \int_{\mathbb{R}^d}-\frac{sgn(z_d+b)}{(1+|z_d+b|)^2}p(z) dz
& = \displaystyle \int_{\mathbb{R}^{d-1}}dz_{1:d-1}
    \left(\int_{-\infty}^{-b}\frac{p(z)}{(1+|z_d+b|)^2} dz_d
    -\int_{-b}^{\infty}\frac{p(z)}{(1+|z_d+b|)^2} dz_d\right)\\
& = \displaystyle \int_{\mathbb{R}^{d-1}}dz_{1:d-1}
    \left(\int_{-\infty}^{0}\frac{p(z-be_d)}{(1+|z_d|)^2} dz_d
    -\int_{0}^{\infty}\frac{p(z-be_d)}{(1+|z_d|)^2} dz_d\right)\\
& = \displaystyle \int_{\mathbb{R}^{d-1}}dz_{1:d-1}
    \int_{0}^{\infty}\frac{p(z+be_d)-p(z-be_d)}{(1+z_d)^2} dz_d\\
\end{array}\]
If $b>0$, then $\|z+be_d\|_2>\|z-be_d\|_2$, so $p(z+be_d)-p(z-be_d)<0$. Similarly, if $b<0$, then $p(z+be_d)-p(z-be_d)>0$. Therefore, we conclude $b=0$, and our objective becomes
\[\int_{\mathbb{R}^d}\frac{p(z)}{1+|z_d|} dz\]

\textbf{Step} 2: \textit{rewriting in polar coordinates.} Let $r=\|z\|_2$, then $z$ can be expressed by polar coordinates in the following form:
\[\left\{
\begin{array}{cl}
    z_1 & =r\sin\theta_1\sin\theta_2\cdots\sin\theta_{d-1} \\
    z_2 & =r\cos\theta_1\sin\theta_2\cdots\sin\theta_{d-1} \\
    z_3 & =r\cos\theta_2\sin\theta_3\cdots\sin\theta_{d-1} \\
    \vdots & \vdots \\
    z_{d-1} & =r\cos\theta_{d-1}\sin\theta_{d-1} \\
    z_d & =r\cos\theta_{d-1} \\
\end{array}
\right. ,\ \theta_1\in[0,2\pi),\theta_i\in[0,\pi),2\leq i\leq d-1\]
The determinant of the Jacobian matrix of this transformation is given below \citep{angel2017ndpolar}:
\[\det J_d=(-1)^{d-1}r^{d-1}\prod_{k=2}^{d-1}\sin^{k-1}\theta_k\]
Therefore, we have
\[\begin{array}{rl}
\displaystyle \int_{\mathbb{R}^d}\frac{p(z)}{1+|z_d|} dz
& \displaystyle = \int_0^{\infty}dr\int_0^{2\pi}d\theta_1\int_0^{\pi}d\theta_2\cdots \int_0^{\pi}d\theta_{d-1}
\left(\frac{p(r)}{1+|r\cos\theta_{d-1}|}r^{d-1}\prod_{k=2}^{d-1}\sin^{k-1}\theta_k\right)\\
& \displaystyle = \left(2\pi\prod_{k=2}^{d-2}\int_0^{\pi} \sin^{k-1}\theta_k d\theta_k \right) \times \int_0^{\infty}dr\int_0^{\pi}d\theta_{d-1}\left(\frac{p(r)}{1+|r\cos\theta_{d-1}|}r^{d-1}\sin^{d-2}\theta_{d-1}\right)
\end{array}\]

\textbf{Step} 3: \textit{further simplification via normalization.}
Since the integration of $p(z)$ over $\mathbb{R}^d$ is $1$, we can write
\[\begin{array}{rl}
\displaystyle 1=\int_{\mathbb{R}^d}p(z)dz
& \displaystyle = \int_0^{\infty}dr\int_0^{2\pi}d\theta_1\int_0^{\pi}d\theta_2\cdots \int_0^{\pi}d\theta_{d-1}
\left(p(r)r^{d-1}\prod_{k=2}^{d-1}\sin^{k-1}\theta_k\right)\\
& \displaystyle = \left(2\pi\prod_{k=2}^{d-1}\int_0^{\pi} \sin^{k-1}\theta_k d\theta_k \right) \times \int_0^{\infty} p(r)r^{d-1} dr
\end{array}\]
Furthermore, notice that
\[\int_0^{\pi}\sin^{d-2}\theta d\theta=\frac{\sqrt{\pi } \Gamma \left(\frac{d-1}{2}\right)}{\Gamma \left(\frac{d}{2}\right)}\]
According to Stirling's formula for Gamma functions, we have
\[\log\Gamma\left(\frac{d-1}{2}\right)-\log\Gamma\left(\frac d2\right)=\Theta\left(\frac{d-1}{2}\log\frac{d-1}{2}-\frac d2\log \frac d2\right)=\Theta\left(-\frac12\log d\right)\]
We are then able to simplify the integration as
\[\int_{\mathbb{R}^d}\frac{p(z)}{1+|z_d|} dz=\frac{\int_0^{\infty}dr\int_0^{\pi}d\theta_{d-1}\left(\frac{p(r)}{1+|r\cos\theta_{d-1}|}r^{d-1}\sin^{d-2}\theta_{d-1}\right)}{\int_0^{\infty} p(r)r^{d-1} dr}\cdot\Theta(\sqrt{d})\]

\textbf{Step} 4: \textit{applying inequalities for two cases.}
\begin{itemize}
    \item When $r\leq d$, we use
        \[\int_0^{\pi}\frac{\sin^{d-2}\theta}{1+|r\cos\theta|} d\theta
        \leq \int_0^{\pi}\sin^{d-2}\theta d\theta=\Theta(d^{-\frac12})\]
    \item When $r>d$, we use
        \[\int_0^{\pi}\frac{\sin^{d-2}\theta}{1+|r\cos\theta|} d\theta\leq \int_0^{\pi}\frac{\sin\theta}{1+|r\cos\theta|} d\theta=\frac{2\log(1+r)}{r}\]
\end{itemize}
Then, we have
\begin{equation}
\label{eqn:lemma: ell_1 1 step 4}
\int_{\mathbb{R}^d}\frac{p(z)}{1+|z_d|} dz=
\frac{\int_0^d p(r)r^{d-1} dr}{\int_0^{\infty} p(r)r^{d-1} dr}
+\Theta(\sqrt{d})\cdot\frac{\int_d^{\infty} p(r)r^{d-2}\log(1+r) dr}{\int_0^{\infty} p(r)r^{d-1} dr}
\end{equation}

\textbf{Step} 5: \textit{final computation.} By applying $p(r)\propto\exp(-r^{\tau})$, we are able to prove the following bounds.
\begin{itemize}
    \item For the first term in \eqref{eqn:lemma: ell_1 1 step 4}, let the incomplete Gamma function be
        \[\gamma(a,x)=\int_0^x t^{a-1}e^{-t} dt\]
        The incomplete Gamma function can be upper bounded below \citep{neuman2013inequalities}:
        \[\gamma(a,x)\leq\frac{x^a(1+ae^{-x})}{a^2}\]
        Therefore, we could bound the first term as
        \[\begin{array}{rl}
            \displaystyle \frac{\int_0^d p(r)r^{d-1} dr}{\int_0^{\infty} p(r)r^{d-1} dr}
            & \displaystyle = \frac{\int_0^d e^{-r^{\tau}}r^{d-1} dr}{\int_0^{\infty} e^{-r^{\tau}}r^{d-1} dr} \\
            (\mbox{let }s=r^{\tau}) & \displaystyle = \frac{\frac{1}{\tau}\int_0^{d^{\tau}} e^{-s}s^{\frac{d}{\tau}-1} ds}{\frac{1}{\tau}\int_0^{\infty} e^{-s}s^{\frac{d}{\tau}-1} ds} \\
            & \displaystyle = \gamma\left(\frac{d}{\tau},d^{\tau}\right) \left/ \Gamma\left(\frac{d}{\tau}\right)\right. \\
            & \displaystyle =\mathcal{O}\left(\frac{\tau^2d^{d-2}+\tau d^{d-1}e^{-d^{\tau}}}{\sqrt{d}\left(\frac{d-\tau}{\tau e}\right)^{\frac{d}{\tau}-1}}\right) \\
            & \displaystyle =\mathcal{O}\left((\tau e)^{\frac{d}{\tau}}d^{-\left(\frac32+\frac{d}{\tau}-d\right)}\right)
        \end{array}\]
    \item For the second term in \eqref{eqn:lemma: ell_1 1 step 4}, we let $\beta$ satisfy $\log(1+d)=d^{\beta}$. Then, $\log(1+r)\leq r^{\beta}$ when $r>d$, and $\beta\rightarrow0$ as $d$ goes to infinity. Thus, we obtain
    \[\begin{array}{rl}
        \displaystyle \frac{\int_d^{\infty} p(r)r^{d-2}\log(1+r) dr}{\int_0^{\infty} p(r)r^{d-1} dr}
        & \displaystyle \leq \frac{\int_d^{\infty} p(r)r^{d+\beta-2} dr}{\int_0^{\infty} p(r)r^{d-1} dr} \\
        & \displaystyle \leq \frac{\int_0^{\infty} p(r)r^{d+\beta-2} dr}{\int_0^{\infty} p(r)r^{d-1} dr} \\
        & \displaystyle = \mathcal{O}\left(\frac{\Gamma\left(\frac{d+\beta-1}{\tau}\right)}{\Gamma\left(\frac{d}{\tau}\right)}\right) \\
        & \displaystyle = \mathcal{O}\left(d^{-\frac{1-\beta}{\tau}}\right) \\
        & \displaystyle = \mathcal{O}\left((\log d)^{\frac{1}{\tau}}d^{-\frac{1}{\tau}}\right)
    \end{array}\]
\end{itemize}

In conclude, the first term in \eqref{eqn:lemma: ell_1 1 step 4} is exponential in $1/d$ while the second term is polynomial in $1/d$. Thus, we finish the proof.
\end{proof}

\subsection*{Proof of Lemma \ref{lemma: ell_1 2}}
\begin{proof}
Similar to the notations in the proof of \textbf{Lemma} \ref{lemma: ell_1 2}, for $r\geq0$, we denote $p(z)$ by $p(r)$ for any $z$ such that $\|z\|_2=r$.

Suppose
\[p(r)\propto\tilde{p}(r)=\left\{
        \begin{array}{cc}
            \exp(-d) & r\leq d^{\frac{1}{\tau}} \\
            \exp(-r^{\tau}) & r > d^{\frac{1}{\tau}}
        \end{array}\right.\]

One can check that $p(r)$ is continuous on $\mathbb{R}$. First, we compute the integration in polar coordinates. Following by steps 2-3 in the proof of  \textbf{Lemma} \ref{lemma: ell_1 1}, we obtain
\[\int_{\mathbb{R}^d}\left(\Delta_{c_h}p\vert_z\right) dz =
\frac{\int_0^{\infty}\left(\Delta_{c_h}p\vert_r\right)r^{d-1}dr}
  {\int_0^{\infty}p(r)r^{d-1}dr} =
\frac{\int_0^{\infty}\left(\Delta_{c_h}\tilde{p}\vert_r\right)r^{d-1}dr}
  {\int_0^{\infty}\tilde{p}(r)r^{d-1}dr}\]

Next, we show that $\left(\Delta_{c_h}\tilde{p}\vert_r\right)/\tilde{p}(r)=\mathcal{O}\left(d^{-\left(\frac{1}{\tau}-1\right)}\right)$. We split the rest of the proof into 3 cases.

\begin{itemize}
    \item When $r\leq d^{1/\tau}-c_h$, $\Delta_{c_h}\tilde{p}\vert_r=0$.

    \item When $r\geq d^{1/\tau}$, the second derivative of $\tilde{p}(r)$ is strictly positive, so
    \[\Delta_{c_h}\tilde{p}\vert_r\leq c_h|\tilde{p}'(r)|=c_h\tau r^{\tau-1}\exp(-r^{\tau})\]
    Therefore,
    \[\frac{\Delta_{c_h}\tilde{p}\vert_r}{\tilde{p}(r)}\leq c_h\tau r^{\tau-1}\leq c_h \tau d^{-\left(\frac{1}{\tau}-1\right)}\]

    \item When $d^{1/\tau}-c_h<r<d^{1/\tau}$, we have
    \[\Delta_{c_h}\tilde{p}\vert_r\leq \Delta_{c_h}\tilde{p}\vert_{d^{1/\tau}}\leq c_h|\tilde{p}'(d^{1/\tau})|=c_h\tau d^{-\left(\frac{1}{\tau}-1\right)}\exp(-d)\]
    Since $\tilde{p}(r)=\exp(-d)$ in this case, we obtain
    \[\frac{\Delta_{c_h}\tilde{p}\vert_r}{\tilde{p}(r)} \leq c_h\tau d^{-\left(\frac{1}{\tau}-1\right)}\]
\end{itemize}
Summing these up, we finish the proof.
\end{proof}

\subsection{Proof of Theorem \ref{thm: ell_1 householder}}
\begin{proof}
Let $f$ be any Householder flow. According to \textbf{Lemma} \ref{lemma: upper bound of L} and the fact that $\det J_f(z)=-1$ for any $z\in\mathbb{R}^d$, we have
\[\mathcal{L}(p,f)\leq\hat{\mathcal{L}}(p,f)=\int_{\mathbb{R}^d}|p(f(z))-p(z)|dz\]
Since a Householder matrix does not change the $\ell_2$ norm of a vector, we have
\[|p(f(z))-p(z)|\leq\sup_{\|x\|_2=\|y\|_2=\|z\|_2}|p(x)-p(y)|\]
Now, we rewrite the integration in polar coordinates (see step 2 in the proof of \textbf{Lemma} \ref{lemma: ell_1 1}), and we obtain that
\[\hat{\mathcal{L}}(p,f) \leq \left(2\pi\prod_{k=2}^{d-1}\int_0^{\pi} \sin^{k-1}\theta_k d\theta_k \right) \times \int_0^{\infty}r^{d-1} \sup_{\|x\|_2=\|y\|_2=r}|p(x)-p(y)| dr\]
First of all,
\[2\pi\prod_{k=2}^{d-1}\int_0^{\pi} \sin^{k-1}\theta_k d\theta_k= \left(2\pi\prod_{k=2}^{d-1}\frac{\sqrt{\pi}\Gamma\left(\frac k2\right)}{\Gamma\left(\frac{k+1}{2}\right)} \right)
= \frac{2\pi^{\frac d2}}{\Gamma\left(\frac d2\right)}\]
Next, we bound $\sup_{\|x\|_2=\|y\|_2=r}|p(x)-p(y)|$ for $r>0$. Let $\Sigma=I+S$. According to the Courant-Fischer theorem (Chapter 5.2.2. (4), \citep{lutkepohl1996handbook}),
\[\begin{array}{rll}
    \displaystyle \max_{\|z\|_2=r} z^{\top}\Sigma^{-1}z
    &\displaystyle = r^2 \lambda_{max}(\Sigma^{-1})
    &\displaystyle = \frac{r^2}{\lambda_{min}(\Sigma)}\\
    \displaystyle \min_{\|z\|_2=r} z^{\top}\Sigma^{-1}z
    &\displaystyle = r^2 \lambda_{min}(\Sigma^{-1})
    &\displaystyle = \frac{r^2}{\lambda_{max}(\Sigma)}
\end{array}\]
Therefore,
\[\sup_{\|x\|_2=\|y\|_2=r}|p(x)-p(y)| = \frac{\exp\left(-\frac{r^2}{2\lambda_{max}(\Sigma)}\right)-\exp\left(-\frac{r^2}{2\lambda_{min}(\Sigma)}\right)}{(2\pi)^{\frac d2}\sqrt{\det\Sigma}}\]
Then. we obtain
\[\begin{array}{rl}
    \displaystyle \int_0^{\infty}r^{d-1} \sup_{\|x\|_2=\|y\|_2=r}|p(x)-p(y)| dr
    &\displaystyle = \frac{1}{(2\pi)^{\frac d2}\sqrt{\det\Sigma}} \int_0^{\infty}r^{d-1} \exp\left(-\frac{r^2}{2\lambda_{max}(\Sigma)}\right)dr\\
    &\displaystyle~~-\frac{1}{(2\pi)^{\frac d2}\sqrt{\det\Sigma}}\int_0^{\infty}r^{d-1}\exp\left(-\frac{r^2}{2\lambda_{min}(\Sigma)}\right)dr \\
    & \displaystyle = \frac{2^{\frac d2-1}\Gamma\left(\frac d2\right)}{(2\pi)^{\frac d2}\sqrt{\det\Sigma}} \left(\lambda_{max}(\Sigma)^{\frac d2}-\lambda_{min}(\Sigma)^{\frac d2}\right)
\end{array}\]
Combining these computations, we have
\[\hat{\mathcal{L}}(p,f)
\leq\frac{2\pi^{\frac d2}}{\Gamma\left(\frac d2\right)}\cdot\frac{2^{\frac d2-1}\Gamma\left(\frac d2\right)}{(2\pi)^{\frac d2}\sqrt{\det\Sigma}} \left(\lambda_{max}(\Sigma)^{\frac d2}-\lambda_{min}(\Sigma)^{\frac d2}\right)
=\frac{\lambda_{max}(\Sigma)^{\frac d2}-\lambda_{min}(\Sigma)^{\frac d2}}{\sqrt{\det\Sigma}}\]
Since $\det\Sigma$ is equal to the product of all eigenvalues of $\Sigma$, we have
\[\begin{array}{rl}
    \hat{\mathcal{L}}(p,f)
    & \displaystyle \leq\left(\frac{\lambda_{max}(\Sigma)}{\lambda_{min}(\Sigma)}\right)^{\frac d2}-1 \\
    & \displaystyle = \left(\frac{\lambda_{max}(S)+1}{\lambda_{min}(S)+1}\right)^{\frac d2}-1 \\
    & \displaystyle = \left(1+\frac{\lambda_{max}(S)-\lambda_{min}(S)}{\lambda_{min}(S)+1}\right)^{\frac d2}-1
\end{array}\]
According to the Gershgorin Circle Theorem, the absolute value of each eigenvalue of $S$ does not exceed $\max_{1\leq i\leq d}\sum_{j=1}^d |S_{ij}|\leq d^{-(1+\kappa)}$. Therefore,
\[\begin{array}{rl}
    \hat{\mathcal{L}}(p,f)
    & \displaystyle \leq \left(1+\frac{2d^{-(1+\kappa)}}{1-d^{-(1+\kappa)}}\right)^{\frac d2}-1 \\
    & \displaystyle = \mathcal{O}\left(\frac{2d^{-(1+\kappa)}}{1-d^{-(1+\kappa)}}\cdot\frac d2\right)\\
    & \displaystyle = \mathcal{O}\left(d^{-\kappa}\right)
\end{array}\]
Finally, by setting $\epsilon=\frac12\|p-q\|_1=\Theta(1)$, we finish the proof.
\end{proof}

\newpage
\section{Experiments}
\subsection{Experiments for \textbf{Theorem} \ref{thm:high d ReLU}}
In Figure \ref{fig: 2d relu topo}, we plot two examples that illustrate \textbf{Theorem} \ref{thm:high d ReLU}. In each example, we plot the surface of $q$ and its transformed distribution $p=f\#q$, where $f$ is a ReLU planar flow. The four peaks of $q$ are marked as red points, and their mapped locations on the surface of $p$ are also marked as red. As illustrated, the mapped locations still correspond to the peaks of $p$, which is consistent with \textbf{Theorem} \ref{thm:high d ReLU} because both $\nabla_z\log q(z)$ and $\nabla_z\log p(f(z))$ are zero vectors.

\begin{figure}[!h]
    \centering
    \fcolorbox{black}{white}{
        \includegraphics[trim=35 35 35 40, clip, width=0.8\textwidth]{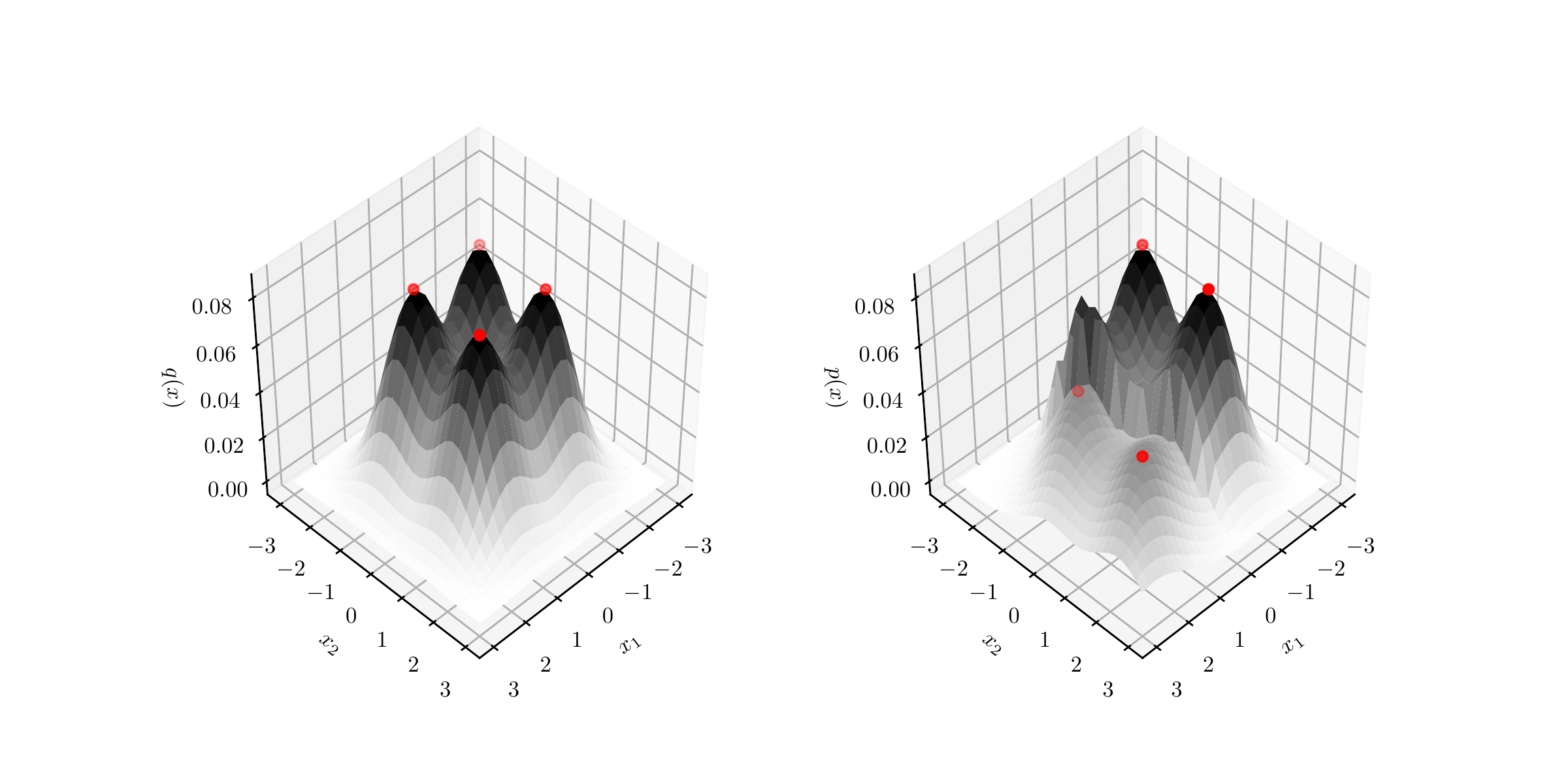}
    }
    \fcolorbox{black}{white}{
        \includegraphics[trim=35 35 35 40, clip,width=0.8\textwidth]{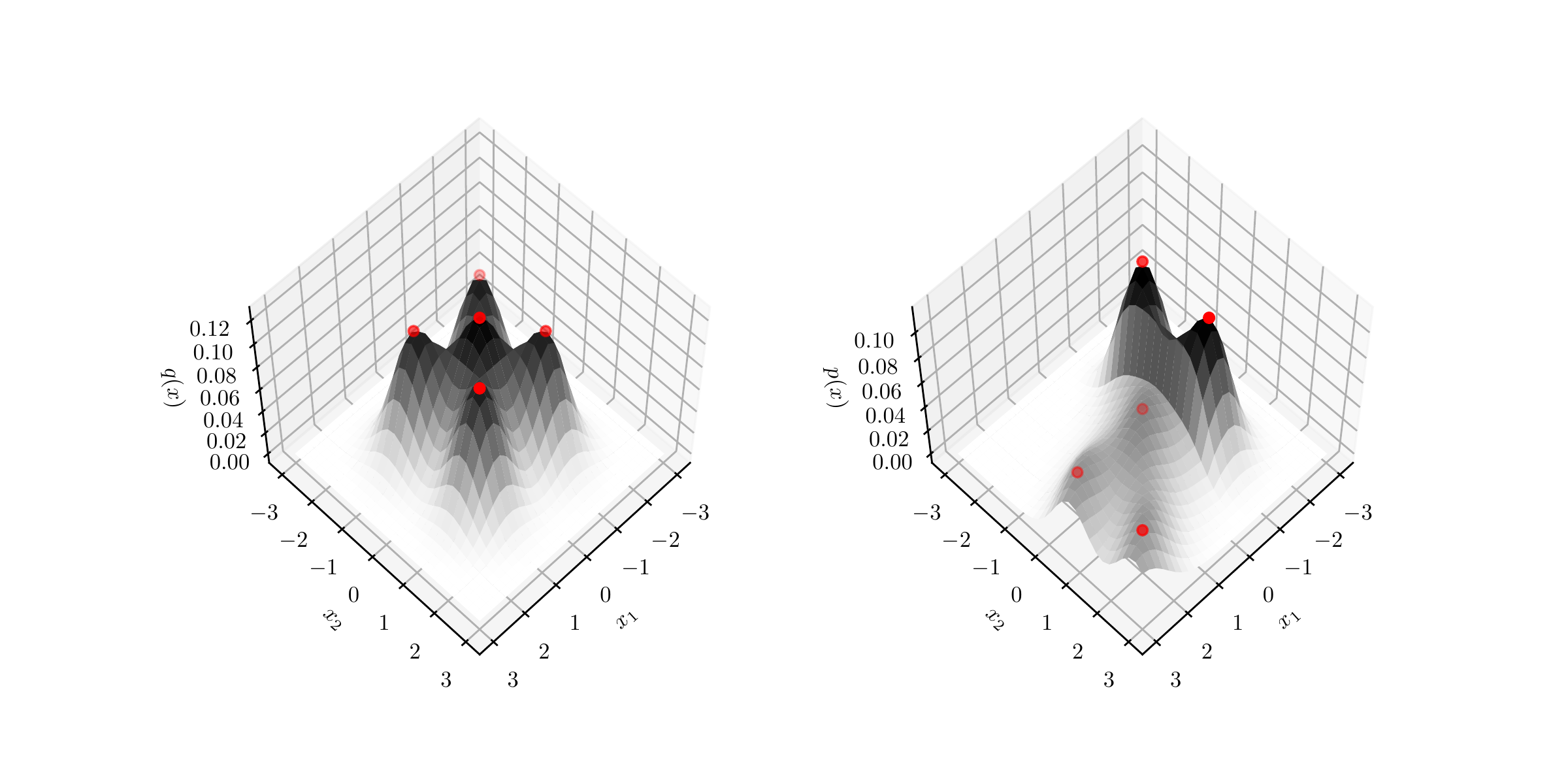}
    }
    \caption{Two examples that illustrate \textbf{Theorem} \ref{thm:high d ReLU}. Each example includes the surface plot of $q$ (left) -- a mixture of Gaussian distribution, and $p=f\#q$ (right) -- the transformed distribution of $q$. The red points correspond to the peaks of $q$ and their mapped points.}
    \label{fig: 2d relu topo}
\end{figure}

\subsection{Experiments for \textbf{Theorem} \ref{thm:high d exact more}}
In Figure \ref{fig: 2d general topo}, we plot four examples that illustrate \textbf{Theorem} \ref{thm:high d exact more}. In each example, we plot the surface of $q$, its transformed distribution $p=f\#q$ where $f$ is a planar flow with non-linearity $\tanh$, and the value $\log p(f(x))-\log q(x)$. As illustrated, the results are consistent with \textbf{Theorem} \ref{thm:high d exact more} because the gradient of $\log p(f(x))-\log q(x)$ is parallel to some constant vector as indicated by applying \textbf{Theorem} \ref{thm:high d exact more} to single planar flow.

\begin{figure}[!h]
    \centering
    \fcolorbox{black}{white}{
        \includegraphics[trim=35 60 35 90, clip,width=0.9\textwidth]{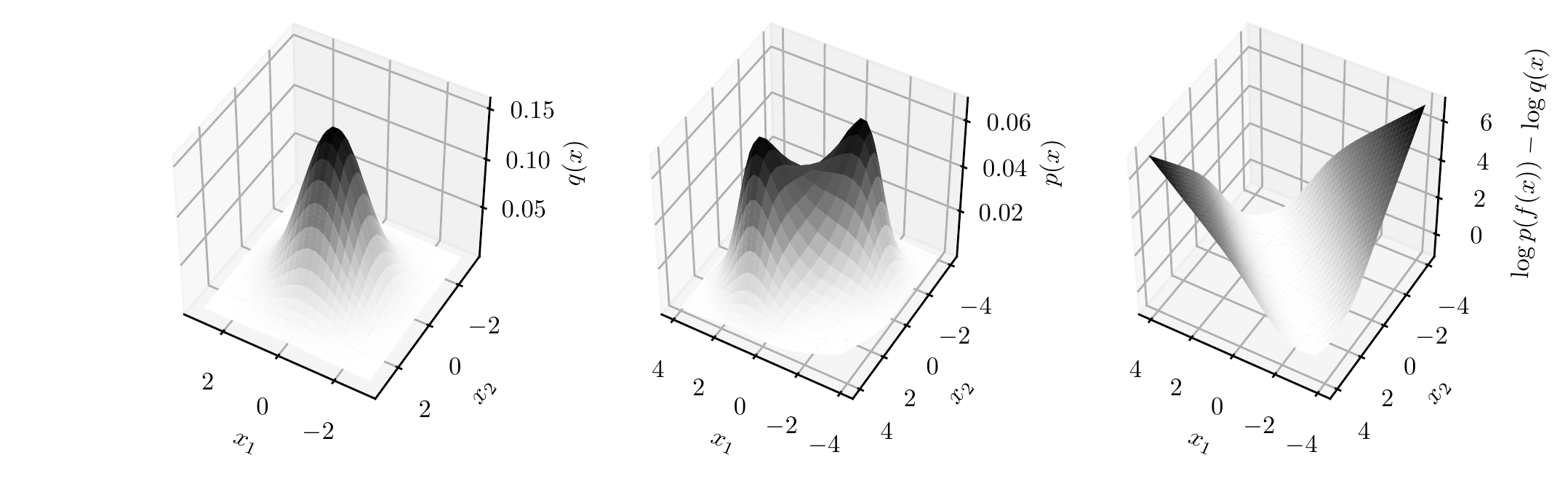}
    }
    \fcolorbox{black}{white}{
        \includegraphics[trim=35 60 35 90, clip,width=0.9\textwidth]{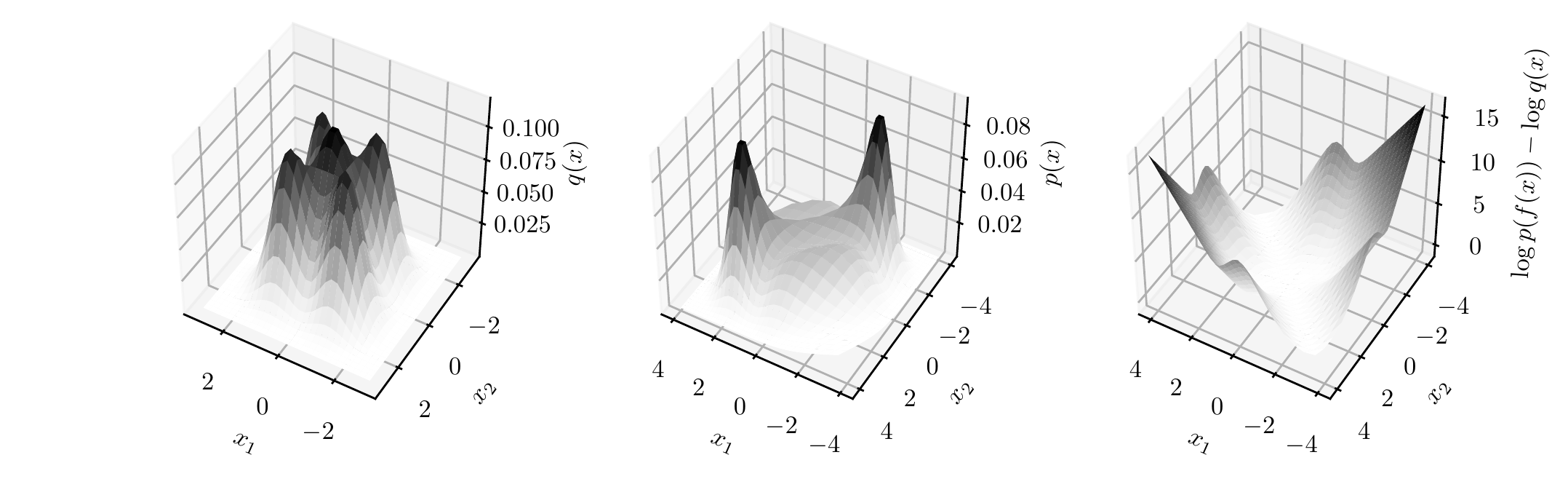}
    }
    \fcolorbox{black}{white}{
        \includegraphics[trim=35 60 35 90, clip,width=0.9\textwidth]{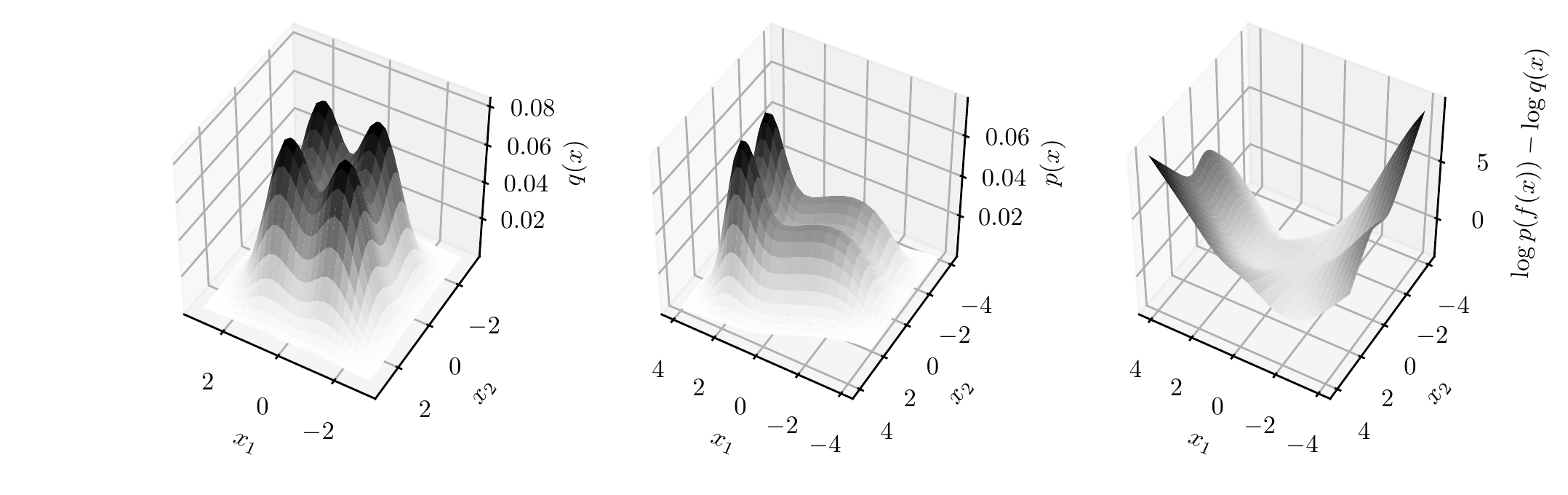}
    }
    \fcolorbox{black}{white}{
        \includegraphics[trim=35 60 35 90, clip,width=0.9\textwidth]{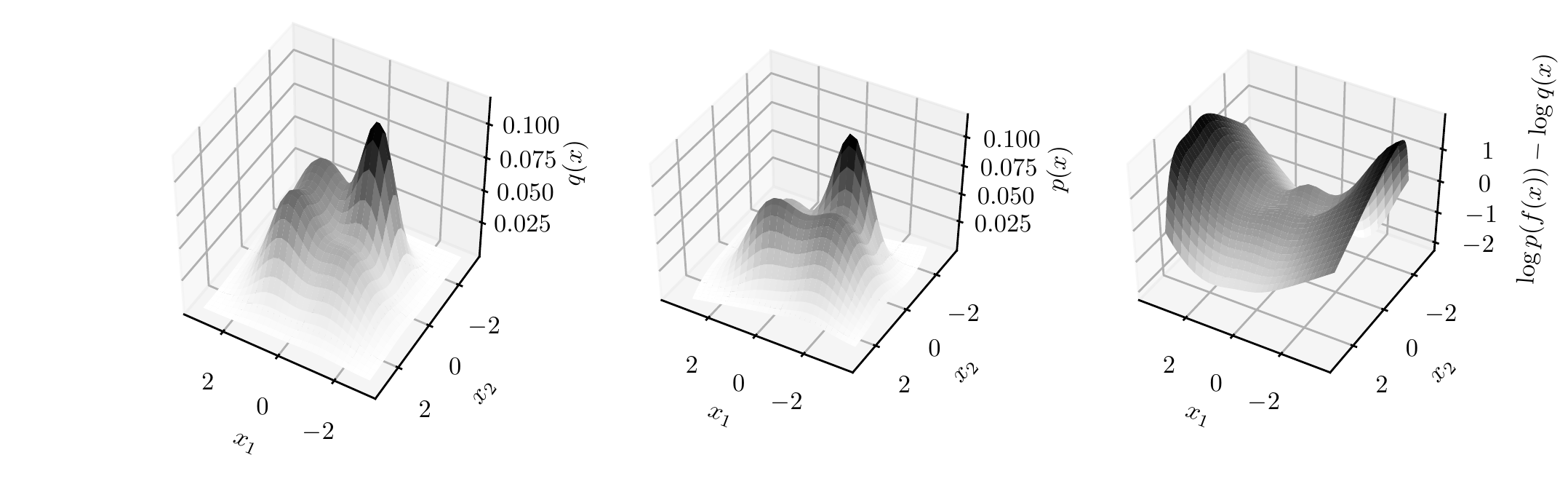}
    }
    \caption{For examples that illustrate \textbf{Theorem} \ref{thm:high d exact more}. Each example includes the surface plot of $q$ (left) -- a mixture of Gaussian distribution, $p=f\#q$ (middle) -- the transformed distribution of $q$ with a planar flow, and $\log p(f(x))-\log q(x)$ (right).}
    \label{fig: 2d general topo}
\end{figure}

\end{document}